\documentclass[twoside,11pt]{article}

\usepackage{jmlr2e}
\usepackage{bm}
\usepackage{bbm}
\usepackage{algorithm}
\usepackage{algorithmic}
\usepackage{amsmath}
\usepackage{color}

\newtheorem{assumption}{\textsc{Assumption}}

\newcommand{\cK}{\mathcal{K}}

\newcommand{\bP}{\mathbb{P}}

\newcommand{\cN}{\mathcal{N}}

\newcommand{\cU}{\mathcal{U}}
\newcommand{\cO}{\mathcal{O}}

\newcommand{\tDelta}{\tilde{\Delta}}

\newcommand{\cC}{\mathcal{C}}
\newcommand{\bE}{\mathbb{E}}

\ShortHeadings{Reward Maximization Under Uncertainty: Leveraging Side-Observations}{Buccapatnam, Liu, Eryilmaz and Shroff}
\firstpageno{1}

\begin{document}

\title{Reward Maximization Under Uncertainty: Leveraging Side-Observations on Networks}

\author{\name Swapna Buccapatnam \email sb646f@att.com \\
       \addr AT$\&$T Labs Research, Middletown, NJ 07748, USA
       \AND
       \name Fang Liu \email liu.3977@osu.edu \\
       \name Atilla Eryilmaz \email eryilmaz.2@osu.edu\\
       \addr Department of Electrical and Computer Engineering\\
       The Ohio State University\\
       Columbus, OH 43210, USA
       \AND
       \name Ness B. Shroff \email shroff.11@osu.edu\\
       \addr Department of Electrical and Computer Engineering and Computer Science Engineering\\
       The Ohio State University\\
       Columbus, OH 43210, USA}

\editor{}

\maketitle

\begin{abstract}
We study the stochastic multi-armed bandit (MAB) problem in the presence of side-observations across actions that occur as a result of an underlying network structure. In our model, a bipartite graph captures the relationship between actions and a common set of unknowns such that choosing an action reveals observations for the unknowns that it is connected to. This models a common scenario in online social networks where users respond to their friends' activity, thus providing side information about each other's preferences. Our contributions are as follows: 1) We derive an asymptotic lower bound (with respect to time) as a function of the bi-partite network structure on the regret of any {\it uniformly good policy} that achieves the maximum long-term average reward. 2) We propose two policies - a randomized policy; and a policy based on the well-known upper confidence bound (UCB) policies - both of which explore each action at a rate that is a function of its network position. We show, under mild assumptions, that these policies achieve the asymptotic lower bound on the regret up to a multiplicative factor, independent of the network structure. Finally, we use numerical examples on a real-world social network and a routing example network to demonstrate the benefits obtained by our policies over other existing policies.
\end{abstract}


\section{Introduction}\label{sec:intro}
Multi-armed bandit (MAB) problems are well-known models of sequential decision-making under uncertainty~\citep{lairobbins} and have lately been used to model new and exciting decision problems in content recommendation systems, online advertising platforms, and
social networks, among others.  In the classical MAB setting, at each time, a bandit policy must choose an action from a set of actions with unknown probability distributions. Choosing an action gives a random reward drawn from the distribution of the action. The regret of any policy is defined as the difference between the total reward obtained from the action with the highest average reward and the given policy's total reward. The goal is to find policies that minimize the expected regret over time.      

In this work, we consider an important extension to the classical MAB problem, where choosing an action not only generates a reward from that action, but also reveals important information for a subset of the remaining actions. We model this relationship between different actions using a bipartite graph between the set of actions and a common set of unknowns (see Figure \ref{fig:examp}). The reward from each action is a known function of a subset of the unknowns (called its parents) and choosing an action reveals observations from each of its parents. Our main objective in this work is to leverage such a structure to improve scalability of bandit policies in terms of the action/decision space.  

Such an information structure between actions becomes available in a variety of applications. For example, consider the problem of \emph{routing} in communication networks, where packets are to be sent over a set of links from source to destination (called a path or a route) in order to minimize the delay. Here, the total delay on each path is the sum of individual link delays, which are unknown. In addition, traveling along a path reveals observations for delays on each of constituent links. Hence, each path provides additional information for all other paths that share some of their links with it. In this example, actions correspond to a set of feasible paths and the set of unknowns corresponds to random delays on all the links in the network. 

Another example occurs in advertising in online social networks through promotional offers. Suppose a user is offered a promotion/discounted price for a product in return for advertising it to his friends/neighbors in an online social network. The influence of the user is then measured by the friends that respond to his message through comments/likes, etc. Each user has an intrinsic unknown probability of responding to such messages on social media. Here, the set of actions correspond to the set of users (to whom promotions are given) and the set of unknowns are the users' intrinsic responsiveness to such promotions.    

In this work, we aim to characterize the asymptotic lower bound on the regret for a general stochastic multi-armed bandit problem in the presence of such an information structure and investigate policies that achieve this lower bound by taking the network structure into account. 
Our main contributions are as follows: 
\begin{itemize} 
\itemsep 0.5mm
\item  We model the MAB problem in the presence of additional structure and derive an asymptotic (with respect to time) lower bound (as a function of the network structure) on the regret of any uniformly good policy which achieves the maximum long term average reward. This lower bound is presented in terms of the optimal value of a linear program (LP).
\item Motivated by the LP lower bound, we propose and investigate the performance of a randomized policy, we call $\epsilon_{t}$-greedy-LP policy, as well as an upper confidence bound based policy, we call UCB-LP policy. Both of these policies {\it explore each action at a rate that is a function of its location in the network.} We show under some mild assumptions that these policies are optimal in the sense that they achieve the asymptotic lower bound on the regret up to a multiplicative constant that is independent of the network structure. 
\end{itemize} 

The model considered in this work is an important first step in the direction of more general models of interdependence across actions. For this model, we show that as the number of actions becomes large, significant benefits can be obtained from policies that explicitly take network structure into account. While $\epsilon_{t}$-greedy-LP policy explores actions at a rate proportional to their network position, its exploration is oblivious to the average rewards of the sub-optimal actions. On the other hand, UCB-LP policy takes into account both the upper confidence bounds on the mean rewards as well as network position of different actions at each time. 

\section{Related Work}\label{sec:relatedSO}
The seminal work of~\cite{lairobbins} showed that the asymptotic lower bound on the regret of any uniformly good policy scales logarithmically with time with a multiplicative constant that is a function of the distributions of actions. Further, \citet{lairobbins} provide constructive policies called Upper Confidence Bound (UCB) policies based on the concept of optimism in the face of uncertainty that asymptotically achieve the lower bound. More recently, \citet{AuerFinite} considered the case of bounded rewards and propose simpler sample-mean-based UCB policies and a decreasing-$\epsilon_{t}$-greedy policy that achieve logarithmic regret uniformly over time, rather than only asymptotically as in the previous works. 

The traditional multi-armed bandit policies incur a regret that is linear in the number of suboptimal arms. This makes them unsuitable in settings such as content recommendation, advertising, etc, where the action space is typically very large. To overcome this difficulty, richer models specifying additional information across reward distributions of different actions have been studied, such as dependent bandits by~\citet{depend}, $\mathcal{X}$-armed bandits by~\citet{xarmed}, linear bandits by~\citet{linear}, contextual side information in bandit problems by~\citet{linUCB}, combinatorial bandits by~\citet{chen2013combinatorial} etc.. 

The works of~\cite{mannor},~\cite{bhagat}, and~\cite{sigmetrics2014} proposed to handle the large number of actions by assuming that choosing an action reveals observations from a larger set of actions. In this setting, actions are embeded in a network and choosing an action provides observations for all the immediate neighbors in the network. The policies proposed in~\cite{mannor} achieve the best possible regret in the adversarial setting~(see \cite{bubeck} for a survey of adversarial MABs) with side-observations, and the regret bounds of these policies are in terms of the independence number of the network. The stochastic version of this problem is introduced in~\cite{bhagat} and~\cite{sigmetrics2014}, which improves upon the results in~\cite{bhagat}. In~\cite{sigmetrics2014}, the authors derive a lower bound on regret in stochastic network setting for any uniformly good policy and propose two policies that achieve this lower bound in these settings up to a multiplicative constant. Our current work extends the setting in~\cite{bhagat,sigmetrics2014} to a more general and important graph feedback structure between the set of actions and a set of common unknowns, which may or may not coincide with the set of actions available to the decision maker. The setting of~\cite{mannor},~\cite{bhagat}, and~\cite{sigmetrics2014} is a special case of this general feedback structure, where the set of unknowns and the set of actions coincide. 

More recently,~\cite{Cohen2016}, have studied the multi-armed bandit problem with a graph based feedback structure similar to~\cite{mannor}, and~\cite{sigmetrics2014}. However, they assume that the graph structure is never fully revealed. In contrast, in many cases such as the problem of routing in communication networks and the problem of influence maximization in social networks, the graph structure is revealed or learnt apriori and is known. When the graph structure is known, the authors in~\cite{sigmetrics2014} propose algorithms for the stochastic setting whose regret performance is bounded by the domination number of the graph. In contrast, the algorithms proposed in~\cite{Cohen2016} assume that the graph is unknown and achieve a regret that is upper bounded by the independence number of the graph. (Note that the independence number of a graph is larger than or equal to the domination number).  Our current work proposes a general feedback structure of which~\cite{sigmetrics2014} and~\cite{Cohen2016} can be viewed as a special case. Moreover, we present algorithms that benefit significantly from the knowledge of the graph feedback structure. 

The setting of combinatorial bandits (CMAB) by~\cite{chen2013combinatorial} is also closely related to our work. In  CMAB, a subset of base actions with unknown distributions form super actions and in each round, choosing a super action reveals outcomes of its constituent actions. The reward obtained is a function of these outcomes. The number of super actions and their composition in terms of base actions is assumed to be arbitrary and the policies do not utilize the underlying network structure between base actions and super actions. In contrast, in our work, we derive a regret lower bound in terms of the underlying network structure and propose policies that achieve this bound. This results in markedly improved performance when the number of super actions is not substantially larger than the number of base actions.  

\section{Problem Formulation}\label{sec:modelSO}
{In this section, we formally define the general bandit problem in the presence of side observations across actions. Let $\mathcal{N}=\{1,\ldots,N\}$ denote the collection of {\it base-arms} with unknown distributions. Subsets of base-arms form {\it actions}, and are indexed by $\mathcal{K}=\{1,\ldots,K\}$. A  decision maker must choose an action $j\in\cK$ at each time $t$ and observes the rewards of related base-arms. Let $X_{i}(t)$ be the reward of base-arm $i$ observed by the decision maker (on choosing some action) at time $t.$ We assume that $\{X_{i}(t), t\ge 0\}$ are independent and identically distributed (i.i.d.) for each $i$ and $\{X_{i}(t), \forall i \in \mathcal{N}\}$ are independent for each time $t.$ Let $V_j\subseteq \cN$ be the subset of base-arms that are observed when playing action $j$. Then, we define $S_i=\{j:i\in V_j\}$ as the support of base-arm $i$, i.e., the decision maker gets observations for base-arm $i$ on playing action $j\in S_i.$ When the decision maker chooses action $j$ at time $t$, he or she observes one realization for each of the random variables $X_i(t)$, $i\in V_j$. The reward of the played action $j$ depends on the outcomes of its related base-arms subset, denoted by $\cK_j\subseteq \cN$, and some known function $f_j(\cdot)$. Note that $\cK_j\subseteq V_j$ because there may be some base-arms that can be observed by action $j$ but not counted as reward in general (see Figure \ref{fig:examp} for a concrete example). Let the vector $\vec{X}_j(t)=[X_i(t)]_{i\in \cK_j}$ denote the collection of random variables associated with the reward of action $j$. Then the reward from playing action $j$ at time $t$ is given by $f_j(\vec{X}_j(t))$. We assume that the reward is bounded in $[0,1]$ for each action. Note that we only assume that the reward function $f_j(\cdot)$ is bounded and the specific form of $f_j(\cdot)$ and $\cK_j$ are determined by the decision maker or the specific problem. Let $\mu_{j}$ be the mean of reward on playing action $j.$}

 {\noindent {\bf Side-observation model :}} 
The actions $\cK$ and base-arms $\cN$ form nodes in a network $G,$ represented by a bipartite graph $(\cK,\cN,E)$ and the collection $\{\cK_j\}_{j\in \cK}$. The $N\times K$ adjacency matrix $E=[e_{ij}]$ is defined by $e_{ij}=1$ if $i\in V_j$ and $e_{ij}=0$ otherwise. If there is an edge between action $j$ and base-arm $i$, i.e., $e_{ij}=1$, then we can observe a realization of base-arm $i$ when choosing action $j$. Intuitively, the bipartite graph determined by $\{V_j\}_{j\in \cK}$ describes the side-observation relationships while the collection $\{\cK_j\}_{j\in \cK}$ captures the reward structure. Without loss of generality, we assume that $\cup_{i\in\cK}\cK_i=\cN$, which means that there are no useless (dummy) unknown base-arms in the network.



 \begin{figure}[h]
\centering
\includegraphics[width=3in]{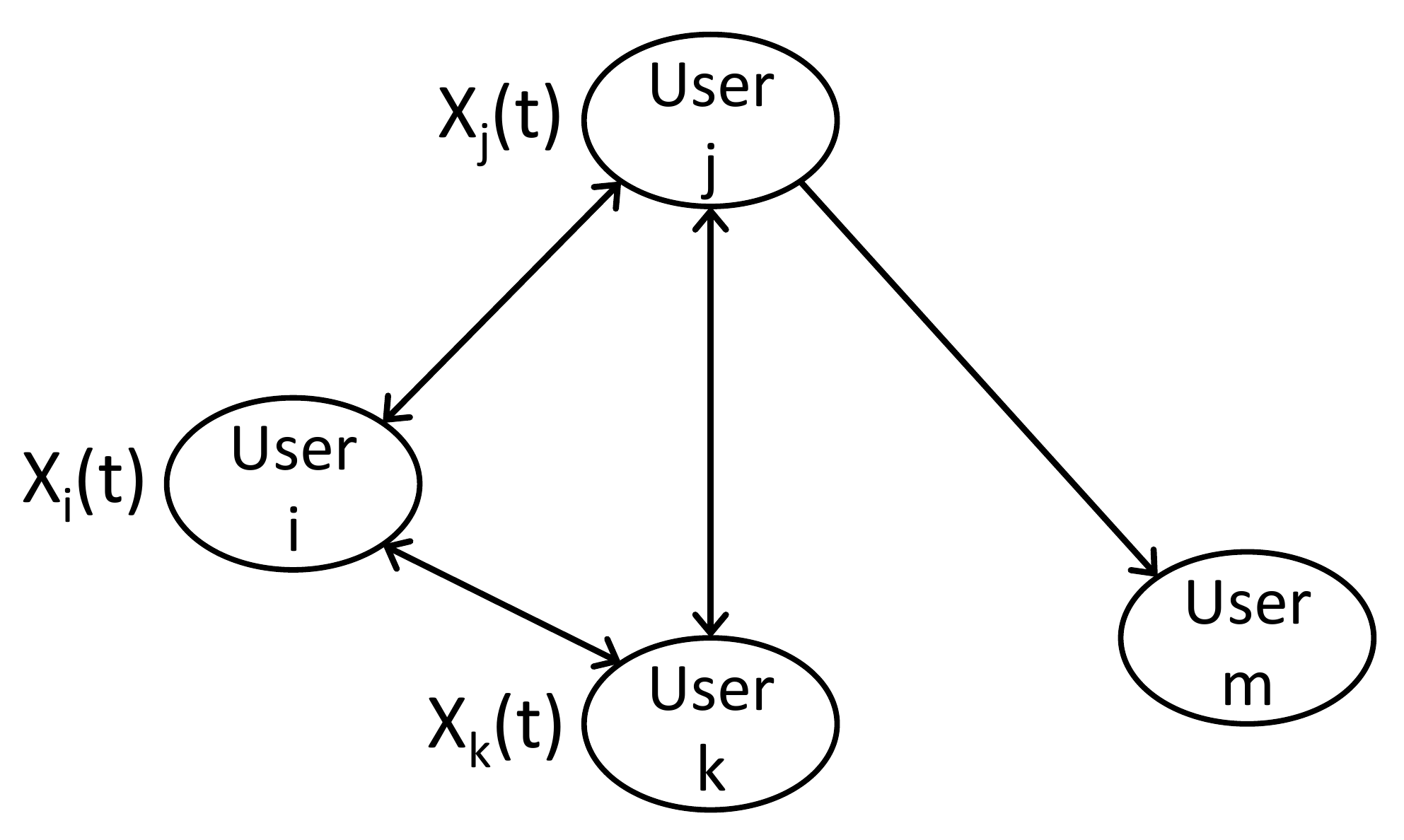}
\caption{At time $t,$ suppose that the decision maker chooses user $i$ to offer a promotion. He then receives a response $X_{i}(t)$ from user $i.$ Using the social interconnections, he also observes responses $X_{j}(t)$ and $X_{k}(t)$ of $i$'s neighbors $j$ and $k.$}
\label{fig:socnet}
\end{figure}
Figure~\ref{fig:socnet} illustrates the side-observation model for the example of targeting users in an online social network. 
Such side observations are made possible in settings of online social networks like Facebook by surveying or tracking a user's neighbors' reactions (likes, dislikes, no opinion, etc.) to the user's activity. This is possible when the online social network has a survey or a like/dislike indicator that generates side observations. For example, when user $i$ is offered a promotion, her neighbors may be queried as follows: ``User $i$ was recently offered a promotion. Would you also be interested in the offer?\footnote{Since, the neighbors do not have any information on whether the user $i$ accepted the promotion, they act independently according to their own preferences in answering this survey. The network itself provides a better way for surveying and obtaining side observations.}'' 

 \begin{figure}[h]
\centering
\includegraphics[width=3in]{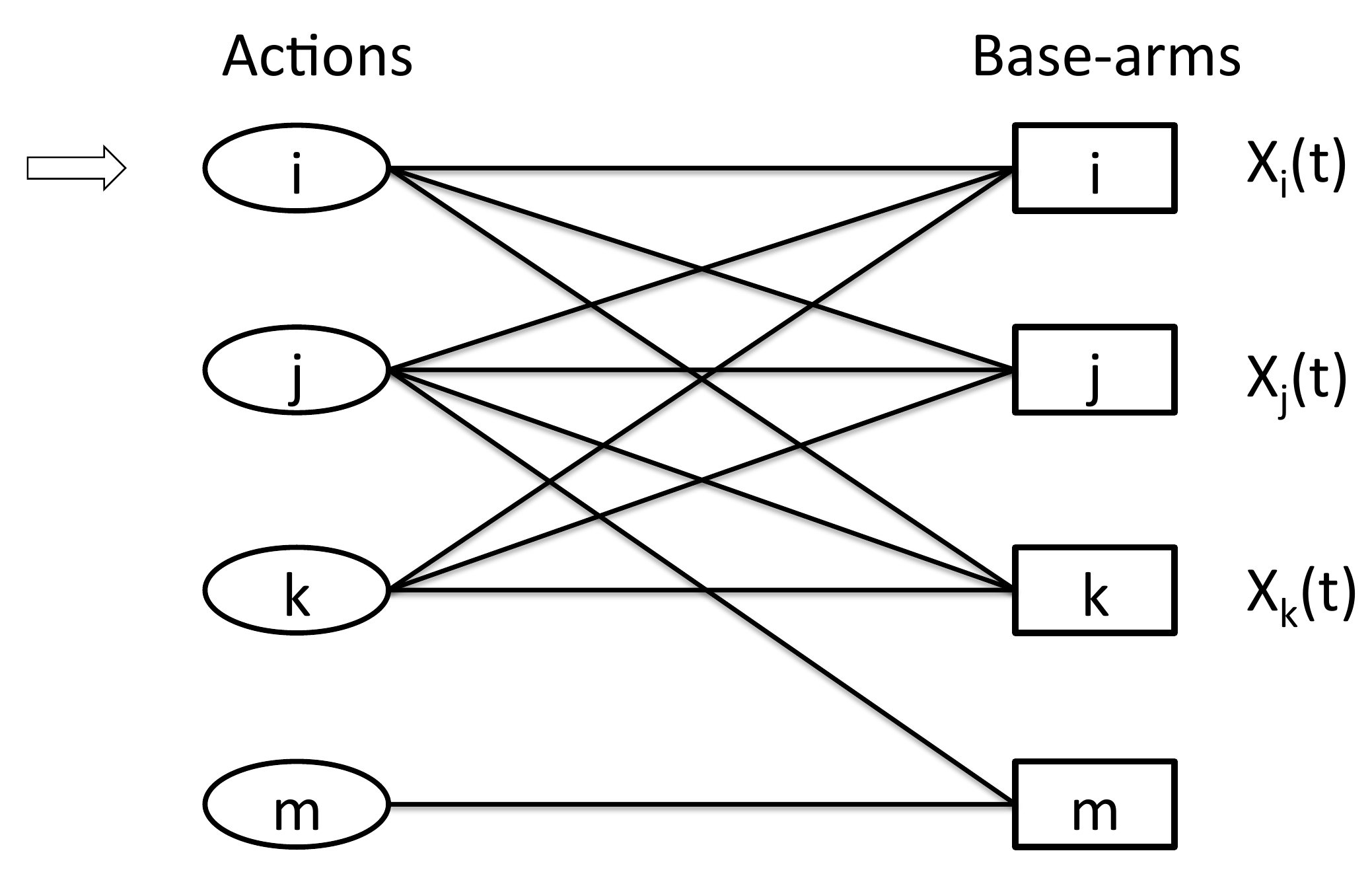}
\caption{Bipartite graph for the example of targeting users in online social network.}
\label{fig:examp}
\end{figure}
{Figure~\ref{fig:examp} shows the bipartite graph generated from the example shown in Figure~\ref{fig:socnet}. The set of base-arms is the set of users since they act independently according to their own preferences in the promotion, which are unknown to the decision maker. The set of actions is also the set of users because the decision maker wants to target the users with the maximum expected reward. When action $j$ (user $j$) is chosen, the decision maker observes $X_i(t)$, $X_j(t)$, $X_k(t)$ and $X_m(t)$ from user $i$, $j$, $k$ and $m$ since $V_j=\{i,j,k,m\}$. The reward of playing action $j$ depends on $\cK_j$ and $f_j(\cdot)$. Suppose $f_j(\vec{X}_j(t))=\sum_{i\in\cK_j}X_{i}(t)$. Then $\cK_j=\{i,j,k,m\}$ means that the reward is the sum of all positive feedbacks. It is also possible that the decision maker set $\cK_j=\{j\}$, which means that the reward of playing action $j$ is only the observation from the user $j$. 

The reward function can be quite general (but bounded) to accommodate different settings. Also, the bipartite graph can be more general than social networks in two key ways: 1) Connecting two users in two-hop neighborhood means that the reaction of the friend of my friend is also observable, which is true in Facebook. 2) Connecting two users, say $i$ and $j$, with similar preference profiles means that the network actively recommends the promotion received by user $i$ to user $j$ even though they are not friends. This has been widely applied in recommender systems such as Yelp.}

\noindent {\bf Objective:}  An allocation strategy or policy $\bm{\phi}$ chooses the action to be played at each time. Formally, $\bm{\phi}$ is a sequence of random variables $\{\phi(t), t \ge 0\},$ where $\phi(t)\in\mathcal{K}$ is the action chosen by policy $\bm{\phi}$ at time $t.$ 
Let $T^{\phi}_{j}(t)$ be the total number of times action $j$ is chosen up to time $t$ by policy $\bm{\phi}.$ For each action, rewards are only obtained when the action is chosen by the policy (side-observations do not contribute to the total reward). Then, the regret of policy $\bm{\phi}$ at time $t$ for a fixed $\bm{\mu}=(\mu_{1},\ldots,\mu_{K})$ is defined by 
\begin{align*}
R^{\phi}_{\bm{\mu}}(t)&=\mu^{*}t-\sum_{j=1}^{K}\mu_{j}\mathbb{E}[T^{\phi}_{j}(t)] =\sum_{j=1}^{K}\Delta_{j}\mathbb{E}[T^{\phi}_{j}(t)],
\end{align*}
where $\Delta_{j}\triangleq\mu^{*}-\mu_{j}$ and $\displaystyle \mu^{*}\triangleq\max_{j\in\cK}\mu_{j}.$ Henceforth, we drop the superscript $\phi$ unless it is required. 
The objective is to find policies that minimize the rate at which the regret grows as a function of time for every fixed network $G.$ We focus our investigation on the class of uniformly good policies~\citep{lairobbins} defined below: 

\noindent {\bf Uniformly good policies:}  An allocation rule $\bm{\phi}$ is said to be uniformly good if for every fixed $\bm{\mu},$ the following condition is satisfied as $t\rightarrow\infty:$ $$\displaystyle R_{\bm{\mu}}(t)=o(t^{b}), \mbox{ for every } b>0.$$ 

The above condition implies that uniformly good policies achieve the optimal long term average reward of $\displaystyle \mu^{*}.$ Next, we define two structures that will be useful later to bound the performance of allocation strategies in terms of the network structure $G.$


\begin{definition}
\label{def:mds}
A {\it hitting set} $D$ is a subset of $\cK$ such that $S_i\cap D\not=\emptyset,~\forall i\in\cN$. Then the hitting set number is $\gamma(G)=\inf_{D\subseteq\cK}\{|D|:S_i\cap D\not=\emptyset,~\forall i\in\cN\}$. For example, the set $\{i,m\}$ is a hitting set in Figure~\ref{fig:examp}.
\end{definition}


\begin{definition}
\label{def:clique}
A {\it clique} $C$ is a subset of $\cK$ such that $\cK_j\subseteq V_i$, $\forall i, j\in C$. This means that for every action $i$ in $C$, we can observe the reward of playing any action $j$ in $C$.
A {\it clique cover} $\cC$ of a network $G$ is a partition of all its nodes into sets $C\in\cC$ such that the sub-network formed by each $C$ is a clique. Let $\bar{\chi}(G)$ be the smallest number of cliques into which the nodes of the network $G$ can be partitioned, also called the clique partition number.
\end{definition}


\begin{proposition}
For any network G with bipartite graph $(\cK,\cN,E)$ and $\{\cK_j\}_{j\in\cK}$, if $\cup_{j\in\cK}\cK_j=\cN$, then $\gamma(G)\leq\bar{\chi}(G)$.
\end{proposition}
\begin{proof}
Let $\mathcal{C}=\{C_1,C_2,...,C_m\}$ be a clique cover with cardinality $m$, i.e., $|\mathcal{C}|=m$ and each $C_k$ is a clique for $k=1,...,m$. Pick arbitrarily an element $a_k$ from $C_k$ for each $k$. Define $\mathcal{H}=\{a_k:k=1,...,m\}$. Now it remains to show that $\mathcal{H}$ is a hitting set, which implies $\gamma(G)\leq\bar{\chi}(G)$. We prove this by contradiction.

Suppose $\mathcal{H}$ is not a hitting set, then $\exists i\in\cN$ s.t. $S_i\cap\mathcal{H}=\emptyset$. Since $\cup_{j\in\cK}\cK_j=\cN$, $\exists j\in\cK$ s.t. $i\in\cK_j$. $\mathcal{C}$ is a clique cover, then $\exists k(j)\in\{1,2,...,m\}$ such that $j\in C_{k(j)}$. By the construction of $\mathcal{H}$, there exists $a_{k(j)}\in\mathcal{H}\cap C_{k(j)}$. By the definition of clique, we have $\cK_j\subseteq V_{a_{k(j)}}$. Thus, we have $a_{k(j)}\in S_i$ since $i\in\cK_j$. It follows that $S_i\cap\mathcal{H}\not=\emptyset$, which contradicts to $S_i\cap\mathcal{H}=\emptyset$. Hence, $\mathcal{H}$ is a hitting set. 
\end{proof}

In the next section, we obtain an asymptotic lower bound on the regret of uniformly good policies for the setting of MABs with side-observations. This lower bound is expressed as the optimal value of a linear program (LP), where the constraints of the LP capture the connectivity of each action in the network.

\section{Regret Lower Bound in the Presence of Side Observations} \label{sec:lb}
In order to derive a lower bound on the regret, we need some mild regularity assumptions~(Assumptions~\ref{eq:klcond1},~\ref{eq:klcond2}, and~\ref{eq:condo}) on the distributions $F_{i}$ (associated with base-arm $i$) that are similar to the ones in~\cite{lairobbins}. Let the probability distribution $F_{i}$ have a univariate density function $g(x;\theta_{i})$ with unknown parameters $\theta_{i}$, for each $i\in\mathcal{N}$. Let $D(\theta||\sigma)$ denote the Kullback Leibler (KL) distance between distributions with density functions $g(x;\theta)$ and $g(x;\sigma)$ and with means $u(\theta)$ and $u(\sigma)$ respectively. 
\begin{assumption} {(Finiteness)}\label{eq:klcond1} We assume that $g(\cdot;\cdot)$ is such that $ 0<D(\theta||\sigma)<\infty$ whenever $u(\sigma)>u(\theta).$
\end{assumption}
\begin{assumption}{(Continuity)}\label{eq:klcond2}
For any $\epsilon>0$ and $\theta,\sigma$ such that $u(\sigma)>u(\theta),$ there exists $\eta>0$ for which $|D(\theta||\sigma)-D(\theta||\rho)|<\epsilon$ whenever $u(\sigma)<u(\rho)<u(\sigma)+\eta.$\end{assumption}
\begin{assumption}{(Denseness)}\label{eq:condo}
For each $i\in\cN,$ $\theta_{i} \in \Theta$ where the set $\Theta$ satisfies: for all $\theta \in\Theta$ and for all $\eta > 0,$ there exists $\theta'\in\Theta$  such that  $u(\theta)<u(\theta')<u(\theta)+\eta.$
\end{assumption}

Let $\vec{\theta}$ be the vector $[\theta_1,\ldots,\theta_N]$. Define $\Theta_{i} =\{\vec{\theta}: \exists k \in {S}_{i} \mbox{ such that } \mu_{k}(\vec{\theta}) < \mu^{*}(\vec{\theta})\}.$ So, not all actions that support base-arm $i$ are optimal. Suppose $\vec{\theta} \in \Theta_{i}.$ For base arm $i,$ define the set $$\mathcal{B}_{i}(\theta_{i})=\{\theta'_{i}: \exists k \in {S}_{i} \mbox{ such that } \mu_{k}(\vec{\theta}'_{i}) >\mu^{*}(\vec{\theta})\},$$ where $\vec{\theta}'_{i}=[\theta_{1},\ldots,\theta'_{i},\ldots\theta_{N}].$  $\vec{\theta}'_{i}$ differs from $\vec{\theta}$ only in the $i^{th}$ parameter. In this set $\mathcal{B}_{i}(\theta_{i}),$ base-arm $i$ contributes towards a unique optimal action. Define constant $J_{i}(\theta_{i}) = \inf\{D(\theta_{i}||\theta'_{i}):\theta'_{i} \in \mathcal{B}_{i}(\theta_{i})\}.$ This is well-defined when $\mathcal{B}_{i}(\theta_{i}) \neq \emptyset.$
 
 The following proposition is obtained using Theorem~2 in~\cite{lairobbins}. It provides an asymptotic lower bound on the regret of any uniformly good policy under the model described in Section~\ref{sec:modelSO}:  
\begin{proposition}\label{prop:lbSO}
Suppose Assumptions~\ref{eq:klcond1},~\ref{eq:klcond2}, and~\ref{eq:condo} hold. Let $\cU=\{j:\mu_{j}<\mu^{*}\}$ be the set of suboptimal actions. Also, let $\Delta_{j}=\mu^{*}-\mu_{j}.$ Then, under any uniformly good policy $\bm{\phi},$ the expected regret is asymptotically bounded below as follows:
\begin{equation}\label{eq:cmu}\displaystyle \liminf_{t\rightarrow\infty}\frac{R_{\bm{\mu}}(t)}{\log(t)}\ge c_{\bm{\mu}},\end{equation}
where $c_{\bm{\mu}}$ is the optimal value of the following linear program (LP) $P_{1}$:
\begin{align*}
P_{1}:\  \min &\sum_{j\in\cU} \Delta_{j} w_{j}, \\
 \mbox{ subject to: } & \sum_{j\in S_{i}}w_{j}\ge \frac{1}{J_i(\theta_{i})},\ \forall i\in \cN,\\
&  w_{j}\ge 0, \ \forall j\in\cK.
\end{align*}
\end{proposition}
\begin{proof}{\it (Sketch)} Let $M_{i}(t)$ be the total number of observations corresponding to base-arm $i$ available at time $t.$ Then, by modifying the proof of  Theorem~2 of~\cite{lairobbins}, we have that, for $i \in \cN,$  $$\displaystyle \liminf_{t\rightarrow\infty}\frac{\mathbb{E}[M_{i}(t)]}{\log(t)}\ge\frac{1}{J_i(\theta_{i})}.$$
An observation is received for base-arm $i$ whenever any action in $S_{i}$ is chosen. Hence, $\displaystyle M_{i}(t)=\sum_{j\in S_{i}}T_{j}(t).$ These two facts give us the constraints in LP $P_{1}.$ See Appendix~\ref{app:lbSO} for the full proof.
\end{proof}
The linear program given in $P_1$ contains the graphical information that governs the lower bound. However, it requires the knowledge of $\vec{\theta}$, which is unknown. This motivates the construction of the following linear program, LP $P_2$, which preserves the graphical structure while eliminating the distributional dependence on $\vec{\theta}.$
\begin{align*}
P_{2}: \min &\sum_{j\in\cK} z_{j}\\
\mbox{subject to: } &\sum_{j\in S_{i}}z_{j} \ge 1, \ \forall i\in \cN,\\
\mbox{and } &z_{j}\ge 0, \ \forall j\in\cK.
\end{align*}
Let $\mathbf{z}^{*}=(z^{*}_{j})_{j\in\cK}$ be the optimal solution of LP $P_{2}.$ In Sections~\ref{sec:epsgreedySO} and~\ref{sec:ucblp}, we use the above LP $P_{2}$ to modify the $\epsilon$-greedy policy in~\cite{AuerFinite} and UCB policy in~\cite{AuerRevised} for the setting of side-observations. We provide regret guarantees of these modified policies in terms of the optimal value $\sum_{j\in \cK}z_{j}^{*}$ of LP $P_{2}.$ We note that the linear program $P_{2}$ is, in fact, the LP relaxation of the minimum hitting set problem on network $G.$ Since, any hitting set of network $G$ is a feasible solution to the LP $P_{2},$ we have that the optimal value of the LP $\sum_{j\in \cK}z_{j}^{*}\le \gamma(G)\le \bar{\chi}(G).$ 

\begin{proposition}\label{prop:lp2opt}
Consider an Erdos-Renyi random bipartite graph $(\cK,\cN,E)$ such that each entry of the matrix $E$ equals $1$ with probability $p$, where $0<p<1$. Suppose $\cup_{j\in\cK}\cK_j=\cN$, i.e., there are no useless base-arms in the network, then $\sum_{j\in \cK}z_{j}^{*}$ is upper-bounded by $\log_{\frac{1}{1-p}}N$ as $N\rightarrow\infty$ in probability.
\end{proposition}
\begin{proof}{\it (sketch)}
Since $\sum_{j\in \cK}z_{j}^{*}\le \gamma(G)$, it remains to be shown that $\gamma(G)$ is upper bounded by the above result. Suppose there are no useless base-arms in the network. Then the set of all actions is a hitting set. Based on this observation, we construct a repeated experiment to generate actions sequentially. Then we define a stopping time $\tau$ as the first time that all the generated actions form a hitting set. Hence, we show the asymptotic result of $\tau$ as the upper bound of $\gamma(G)$. See full proof in Appendix~\ref{app:lp2opt}.
\end{proof}

In the next proposition, we provide a lower bound on $c_{\bm{\mu}}$ in Equation~(\ref{eq:cmu}) using the optimal solution $\mathbf{z}^{*}=(z^{*}_{j})_{j\in\cK}$ of LP $P_{2}.$
\vspace{0.2cm}
\begin{proposition}\label{prop:lbk}
Let $\cU=\{j:\mu_{j}<\mu^{*}\}$ be the set of suboptimal actions. Let $\cO=\{j:\mu_{j}=\mu^{*}\}$ be the set of optimal actions. Then, \begin{equation}\label{eq:lb}\frac{\max_{i\in\cN} J_i(\theta_i)}{\min_{j\in\cU}\Delta_j}c_{\bm{\mu}}+|\cO|\ge\sum_{j\in\cK}z^*_j\ge \frac{\min_{i\in\cN} J_i(\theta_i)}{\max_{j\in\cU}\Delta_j}c_{\bm{\mu}}.\end{equation} 
\end{proposition}
\begin{proof}{\it (Sketch)} 
Using the optimal solution of LP $P_{1},$ we construct a feasible solution satisfying constraints in LP $P_{2}$ for base-arms in $\cN$.
The feasible solution constructed in this way gives an upper bound on the optimal value of LP $P_{2}$ in terms of the optimal value of LP $P_{1}.$ For the lower bound, we use the fact that any feasible solution of $P_{2},$ in particular $\mathbf{z}^{*},$ can be used to construct a feasible solution of $P_{1}.$ See Appendix~\ref{app:lbk} for the full proof.
\end{proof}
We note that $\sum_{j\in\cK}z^*_j =\Theta(c_{\mu})$ completely captures the time dependence of the regret on network structure under the following assumption:
\begin{assumption}\label{assum:bounded}
The quantities $|\cO|,$ $\displaystyle \min_{j\in\cU}\Delta_{j},$ and $\displaystyle \min_{i\in\cN} J_i(\theta_i)$ are constants that are independent of network size $K$ and $N$.
\end{assumption}
Note that the constants in the above assumption are unknown to the decision maker. In the next section, we propose the $\epsilon_{t}$-greedy-LP policy which achieves the regret lower bound of $c_{\bm{\mu}}\log(t)$ up to a multiplicative constant factor that is independent of the network structure and time.

\section{Epsilon-greedy-LP policy} \label{sec:epsgreedySO}
{Motivated by the LPs $P_{1}$ and $P_{2},$ we propose a \emph{network-aware} randomized policy called the $\epsilon_{t}$-greedy-LP policy. 
We provide an upper bound on the regret of this policy and show that it achieves the asymptotic lower bound, up to a constant multiplier, independent of the network structure. 
Let $\bar{f}_{j}(t)$ be the empirical average of observations (rewards and side-observations combined) available for action $j$ up to time $t.$ The $\epsilon_{t}$-greedy-LP policy is described in Algorithm~\ref{alg:epsgreedyk}. The policy consists of two iterations - exploitation and exploration, where the exploration probability decreases as $1/t,$ similarly to that of the $\epsilon_{t}$-greedy policy proposed by~\citet{AuerFinite}. However, in our policy, we choose the exploration probability for action $j$ to be proportional to $z_{j}^{*}/t,$ where $\mathbf{z}^{*}$ is the optimal solution of LP $P_{2},$ while in the original policy in~\citet{AuerFinite}, the exploration probability is uniform over all actions. 
}

\begin{algorithm}[ht]
\caption{: $\epsilon_{t}$-greedy-LP}
\begin{algorithmic}
\label{alg:epsgreedyk}
\item[]
\item[{\bf Input}:] $c>0,$ $0<d<1,$ optimal solution $\mathbf{z} ^{*}$ of LP $P_{2}.$
\FOR{each time $t$}
\vspace{0.1cm}
\item[] Update $\bar{f_j}(t)$ for each $j\in\cK$, where $\bar{f_j}(t)$ is the empirically average over all the observations of action $j$.
\item[] Let $\displaystyle \epsilon(t)= \min\left(1,\frac{c\sum_{j\in\cK} z_{j}^{*}}{d^{2}t}\right)$ and $\displaystyle a^{*}=\arg\max_{j\in\cK}\bar{f_j}(t)$.
\vspace{0.1cm}
\item[] Sample $a$ from the distribution such that $\displaystyle \mathbb{P}\{a=j\}=\frac{z_j^*}{\sum_{i\in\cK} z_{i}^{*}}$ for all $j\in\cK$.
\item[] Play action $\phi(t)$ such that
\begin{eqnarray}\label{phi_t}
\phi(t)=
\begin{cases}
a, &\text{with probability } \epsilon(t)\cr
a^*, &\text{with probability } 1-\epsilon(t)
\end{cases}
\end{eqnarray}
\ENDFOR
\end{algorithmic}\end{algorithm}
The following proposition provides performance guarantees on the $\epsilon_{t}$-greedy-LP policy:
\begin{proposition}\label{prop:epsgreedySO}
For $\displaystyle 0<d<\min_{j\in\cU}\Delta_{j},$ any $c>0,$ and $\alpha>1,$ the probability with which a suboptimal action $j$ is selected by the $\epsilon_{t}$-greedy-LP policy, described in Algorithm~\ref{alg:epsgreedyk}, for all $\displaystyle t>t'=\frac{c\sum_{i\in\cK}z_{i}^{*}}{d^{2}}$ is at most
\begin{equation}\label{eq:epsub}\left(\frac{c}{d^{2}t}z_{j}^{*}\right)+\frac{2\lambda c \delta}{\alpha d^2}\left(\frac{et^\prime}{t}\right)^{cr/\alpha d^2} \log \left(\frac{e^2t}{t^\prime}\right)+\frac{4}{\Delta_j^2}\left(\frac{et^\prime}{t}\right)^{\frac{c\Delta_j^2}{2\alpha d^2}},\end{equation}
where $r=\frac{3(\alpha-1)^{2}}{8\alpha-2},$ $\lambda=\max_{j\in\cK}|\cK_j|$, and $\delta=\max_{i\in \cN}|S_i|$ is the maximum degree of the supports in the network. Note that $\alpha$ is a parameter we introduce in the proof, which is used to determine a range for the choice of parameter $c$ as shown in Corollary~\ref{prop:epsgreedyk}.
\end{proposition}
\begin{proof}{\it(Sketch)} Since $\mathbf{z}^{*}$ satisfies the constraints in LP $P_{2},$ there is sufficient exploration within each suboptimal action's neighborhood. The proof is then a combination of this fact and the proof of Theorem~3 in~\cite{AuerFinite}. In particular, we derive an upper bound for the probability that suboptimal action $j$ is played at each time and then sum over the time. See Appendix~\ref{app:epsgreedySO} for the full proof.
\end{proof}
In the above proposition, for large enough $c,$ we see that the second and third terms are $O(1/t^{1+\epsilon})$ for some $\epsilon>0$~\citep{AuerFinite}. Using this fact, the following corollary bounds the expected regret of the $\epsilon_{t}$-greedy-LP policy:
\begin{corollary}\label{prop:epsgreedyk}
Choose parameters $c$ and $d$ such that, $$\displaystyle 0<d<\min_{j\in\cU}\Delta_{j},\quad  \mbox{and}\quad \displaystyle c>\max(2 \alpha d^{2}/r,4\alpha),$$ for any $\alpha>1.$ Then, the expected regret at time $T$ of the $\epsilon_{t}$-greedy-LP policy described in Algorithm~\ref{alg:epsgreedyk} is at most
\begin{equation}\label{eq:epsubk}\left(\frac{c}{d^{2}}\sum_{j\in\cU}\Delta_{j}z_{j}^{*}\right)\log(T)+O(K),\end{equation}
where the $O(K)$ term captures constants independent of time but dependent on the network structure. In particular, the $O(K)$ term is at most
\begin{equation*}
\sum_{j\in\cU}\left[\frac{\pi^2\lambda c\delta\Delta_j}{3\alpha d^2}\left(et'\right)^{cr/\alpha d^2} + \frac{2\pi^2}{3\Delta_j}\left(et'\right)^{\frac{c\Delta_j^2}{2\alpha d^2}}\right],
\end{equation*}
where $t^\prime$, $r$, $\lambda$ and $\delta$ are defined in Proposition~\ref{prop:epsgreedySO}.
\end{corollary}
\begin{remark} Under Assumption~\ref{assum:bounded}, we can see from Proposition~\ref{prop:lbk} and Corollary~\ref{prop:epsgreedyk} that, $\epsilon_{t}$-greedy-LP algorithm is \emph{order optimal} achieving the lower bound $\displaystyle \Omega\left(\sum_{j\in\cK}z^{*}_{j}\log(T)\right)=\Omega\left(c_{\bm{\mu}}\log(T)\right)$ as the network and time scale.  
\end{remark}
While the $\epsilon_{t}$-greedy-LP policy is network aware, its exploration is oblivious to the observed average rewards of the sub-optimal actions. Further, its performance guarantees depend on the knowledge of $\min_{j\in\cU} \Delta_{j},$ which is the difference between the best and the second best optimal actions. On the other hand, the UCB-LP policy proposed in the next section is network-aware taking into account the average rewards of suboptimal actions. This could lead to better performance compared to $\epsilon_{t}$-greedy-LP policy in certain situations, for example, when the action with greater $z_j^*$ is also highly suboptimal.
\section{UCB-LP policy}\label{sec:ucblp}
In this section we develop the UCB-LP policy defined in Algorithm~\ref{alg:ucblp} and obtain upper bounds on its regret. The UCB-LP policy is based on the improved UCB policy proposed in~\citet{AuerRevised}, which can be summarized as follows: the policy estimates the values of $\Delta_{i}$ in each round by a value $\tilde{\Delta}_{m}$ which is initialized to 1 and halved in each round $m$. By each round $m$, the policy draws $n(m)$ observations for each action in the set of actions not eliminated by round $m,$ where $n(m)$ is determined by $\tilde{\Delta}_{m}.$ Then, it eliminates those actions whose UCB indices perform poorly. Our policy differs from the one in~\cite{AuerRevised} by accounting for the presence of side-observations - this is achieved by choosing each action according to the optimal solution of LP $P_{2},$ while ensuring that $n(m)$ observations are available for each action not eliminated by round $m.$ 

\begin{algorithm}[H]
\caption{: UCB-LP policy}
\begin{algorithmic}
\label{alg:ucblp}
\item[{\bf Input}:] Set of actions $\cK,$ time horizon $T,$ and optimal solution $\mathbf{z} ^{*}$ of LP $P_{2}.$
\vspace{0.1in}
\item[{\bf Initialization}:] Let $\tilde{\Delta}_{0}:=1,$ $A_{0}:=\cK,$ and $B_{0}:=\cK$
\FOR{round $ m=0,1,2,\ldots,\lfloor\frac{1}{2}\log_{2}\frac{T}{e}\rfloor$} 
\item[] {\bf Action Selection:} Let $\displaystyle n(m):=\left\lceil \frac{2\log(T\tilde{\Delta}_{m}^{2})}{\tilde{\Delta}_{m}^{2}}\right\rceil$ \\
\vspace{0.1in}
\item[] \textbf{If} $|B_{m}| =1$: choose the single action in $B_{m}$ until time $T.$
\vspace{0.1in}
\item[] \textbf{Else If} $\displaystyle\sum_{i\in \cK}z_{i}^{*} \leq 2|B_{m}|\tDelta_{m}:$ $\forall j\in A_m$, choose action $j$ $\left[z_{j}^{*}(n(m)-n(m-1))\right]$ times.
\vspace{0.1in}
\item[] \textbf{Else} For each action $j$ in $B_{m},$ choose $j$ for $\left[n(m)-n(m-1)\right]$ times.\\
\vspace{0.1in}
\item[] Update $\bar{f}_{j}(m)$ and $T_{j}(m)$ for each $j\in\cK$, where $\bar{f}_{j}(m)$ is the empirical average reward of action $j,$ and $T_{j}(m)$ is the total number of observations for action $j$ up to round $m$.

\vspace{0.1in}
\item[] {\bf Action Elimination:} \\ To get $B_{m+1},$ delete all actions $j$ in $B_{m}$ for which
$$\bar{f}_{j}(m)+\sqrt{\frac{\log(T\tilde{\Delta}_{m}^{2})}{2T_{j}(m)}}< \max_{a \in B_{m}}\left\{\bar{f}_{a}(m)-\sqrt{\frac{\log(T\tilde{\Delta}_{m}^{2})}{2T_{a}(m)}}\right\},$$
\vspace{0.1in}
\item[] {\bf Reset:} \\ The set $A_{m+1}$ is given as $A_{m+1}=\bigcup_{i\in D_{m+1}} S_i$, where $D_{m+1}=\bigcup_{j\in B_{m+1}}\mathcal{K}_j$.\\ Let $\tilde{\Delta}_{m+1}=\frac{\tilde{\Delta}_{m}}{2}.$ \\
\vspace{0.1in}
\ENDFOR
\end{algorithmic}
\end{algorithm}

The following proposition provides performance guarantees on the expected regret due to UCB-LP policy:
\begin{proposition}\label{prop:ucblp} 
For action $j,$ define round $m_{j}$ as follows: $$m_j:=\min\left\{m : \tilde{\Delta}_{m}<\frac{\Delta_j}{2}\right\}.$$
Define $\displaystyle \bar{m}=\min\left\{m: \sum_{j\in \cK}z_{j}^{*} > \sum_{j:m_{j}>m}2^{-m+1}\right\}$ and the set  $\displaystyle B =\{j \in \cU:m_{j}>\bar{m}\}.$
Then, the expected regret due to the UCB-LP policy described in Algorithm~\ref{alg:ucblp} is at most
\begin{equation}\label{eq:ucblp} \sum_{j\in \cU\setminus B} \Delta_{j} z_{j}^{*}\frac{32\log(T\hat{\Delta}_{j}^{2})}{\hat{\Delta}_{j}^{2}} + \sum_{j\in B} \frac{32\log(T\Delta_{j}^{2})}{\Delta_{j}}+O(K^{2}),\end{equation}
where $\hat{\Delta}_{j}=\max\{2^{-\bar{m}+2},\min_{a:j\in G_a}\{\Delta_{a}\}\}$, $G_a=\cup_{i\in\cK_a}S_i$, and $(z_{j}^{*})$ is the solution of LP $P_{2}.$ The $O(K^{2})$ term captures constants independent of time. Further, under Assumption~\ref{assum:bounded}, the regret is also at most \begin{equation}\label{eq:ucblp2}O\left(\sum_{j\in\cK}z_{j}^{*}\log(T)\right) + O(K^{2}),\end{equation} where $(z_{j}^{*})$ entirely captures the time dependence on network structure.
\end{proposition}
\begin{proof}{\it(Sketch)}
The $\log(T)$ term in the regret follows from the fact that, with high probability, each suboptimal action $j$ is eliminated (from the set $B_{m}$) on or before the first round $m$ such that $\tilde{\Delta}_{m}<\Delta_{j}/2.$ See Appendix~\ref{app:ucblpSO} for the full proof.
\end{proof}

\begin{remark} While $\epsilon_{t}$-greedy-LP does not require knowledge of the time horizon $T,$ UCB-LP policy requires the knowledge of $T.$ UCB-LP policy can be extended to the case of an unknown time horizon similar to the suggestion in~\cite{AuerRevised}. Start with $T_{0}=2$ and at end of each $T_{l},$ set $T_{l+1}=T_{l}^{2}.$ The regret bound for this case is shown in Proposition~\ref{app:unknown} in Appendix~\ref{sec:appendixSig}.
\end{remark}

Next, we briefly describe the policies UCB-N and UCB-MaxN proposed in~\citet{bhagat}. In the UCB-N policy, at each time, the action with the highest UCB index is chosen similar to UCB1 policy in~\citet{AuerFinite}. In UCB-MaxN policy, at each time $t$, the action $i$ with the highest UCB index is identified and its neighboring action $j$ with the highest empirical average reward at time $t$ is chosen.
\begin{remark}
\label{rem:diff}
The regret upper bound of UCB-N policy is 
$$\inf_{\cC} \sum_{C\in\cC} \frac{8\max_{j\in C}\Delta_{j}}{\min_{j\in C}\Delta_{j}^{2}}\log(T) + O(K),$$ where $\cC$ is a clique covering of the sub-network of suboptimal actions. The regret upper bound for UCB-MaxN is the same as that for UCB-N with an $O(|\cC|)$ term instead of the time-invariant $O(K)$ term. We show a better regret performance for UCB-LP policy and $\epsilon_{t}$-greedy-LP policies with respect to the $\log(T)$ term because $\sum_{i\in \cK}z_{i}^{*}\le \gamma(G)\le\bar{\chi}(G).$ However, the time-invariant term in our policies is $O(K)$ and $O(K^{2}),$ can be worse than the time-invariant term $O(|\cC|)$ in UCB-MaxN.

\end{remark}

\begin{remark}
All uniformly good policies that ignore side-observations incur a regret that is at least $\Omega(|\cU|\log(t))$~\cite{lairobbins}, where $|\cU|$ is the number of suboptimal actions. This could be significantly higher than the guarantees on the regret of both $\epsilon_{t}$-greedy-LP policy and UCB-LP policy for a rich network structure as discussed in Remark~\ref{rem:diff}. 
\end{remark}

\begin{remark}\label{rem:stochasticso} In our model, we assumed that the side observations are always available. However, in reality, side observations may only be obtained sporadically. Suppose that when action $j$ is chosen, side-observations of base-arms $i\in\cK_j$ are obtained almost surely and that of base-arms $i\in V_j\setminus \cK_j$ are obtained with a known probability $p_{j}.$ 
In this case, Proposition~\ref{prop:lbSO} holds with the replacement of LP $P_{1}$ with LP $P'_{1}$ as follows:
\begin{align*}
P'_{1}:\  \min &\sum_{j\in\cU} \Delta_{j} w_{j}, \\
 \mbox{ subject to: } & \sum_{j\in S_{i}}(w_{j}\mathbbm{1}_{\{i\in \cK_j\}}+p_jw_j\mathbbm{1}_{\{i\not\in\cK_j\}})\ge \frac{1}{J_i(\theta_{i})},\ \forall i\in \cN,\\
&  w_{j}\ge 0, \ \forall j\in\cK.
\end{align*}
Both of our policies work for this setting by changing the LP $P_{2}$ to $P'_{2}$ as follows: 
\begin{align*}
P'_{2}:\  \min &\sum_{j\in\cK} z_{j}\\
\mbox{subject to: } &\sum_{j\in S_{i}}(z_{j}\mathbbm{1}_{\{i\in \cK_j\}}+p_jz_j\mathbbm{1}_{\{i\not\in\cK_j\}})\ge 1, \ \forall i\in\cN,\\
\mbox{and } &z_{j}\ge 0, \ \forall j\in\cK.
\end{align*}
The regret bounds of our policies will now depend on the optimal solution of LP $P'_{2}.$ \\
\end{remark}
\section{Numerical Results}\label{sec:numerical}
\subsection{Algorithm Performance on Data Trace}
We consider the Flixster network dataset for the numerical evaluation of our algorithms. The authors in~\cite{jamali} collected this social network data, which contains about $1$ million users and $14$ million links. We use graph clustering by~\cite{Graclus} to identify two strongly clustered sub-networks of sizes $1000$ and $2000$ nodes. Both these sub-networks have a degree distribution that is a straight line on a log-log plot indicating a power law distribution commonly observed in social networks.~\footnote{We note that the social network of interest may or may not display a power law behavior. We find that the subgraphs of the Flixster network have a degree distribution that is a straight line on a log-log plot indicating a power law distribution display while the authors in~\cite{facebook} show that the degree distribution of the global Facebook network is not a straight line on log-log plot.}

Our empirical setup is as follows. Let $\cN$ be the set of users and $\cK=\cN$ . To be specific, each user in the network is offered a promotion at each time, and accepts the promotion with probability $\mu_{i}\in[0.3,0.9].$ Let $S_i$ be the set of one-hop neighbors in the social network of user $i$ (including user $i$). This is the setting when the Flixster has a survey or a like/dislike indicator that generates side observations of user's neighborhood. Let $\cK_j=\{j\}$ and $f_j(X_j)=X_j$, which means that the decision maker receives a random reward of $1$ if the chosen user $j$ accepts the promotion or $0$ reward otherwise. $\mu_{j}$ is chosen uniformly at random from $[0.3,0.8]$ and there are $50$ randomly chosen users with optimal $\mu_{j}=0.9.$ 

 \begin{figure}[ht]
\centering
\includegraphics[width=4.0in]{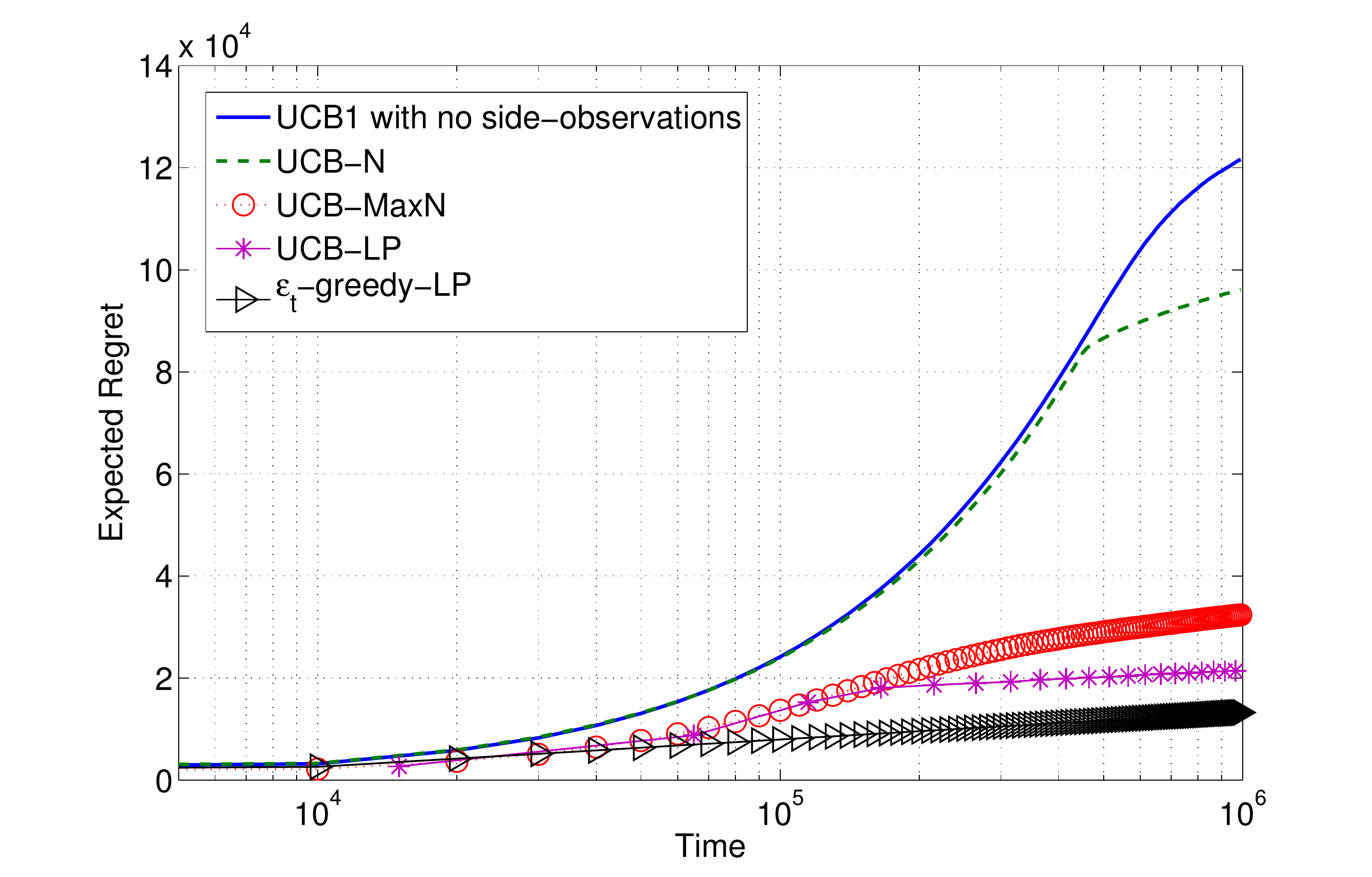}
\caption{Regret comparison of all the policies for a network of size $1000.$ }
\label{fig:1000}
\end{figure}
\begin{figure}[ht]
\centering
\includegraphics[width=4.0in]{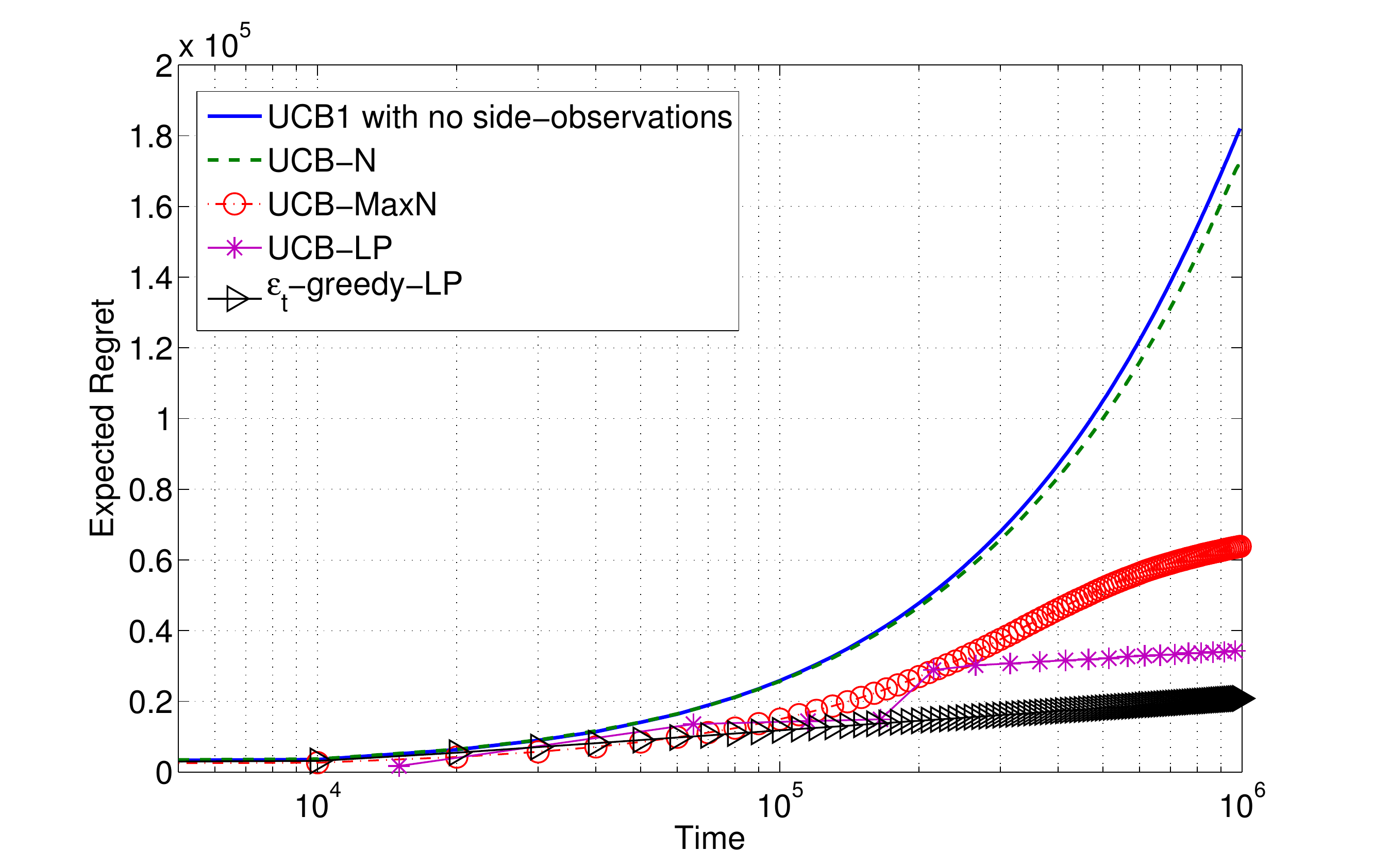}
\caption{Regret comparison of all the policies for a network of size $2000.$}
\label{fig:2000}
\end{figure}

Figures~\ref{fig:1000} and~\ref{fig:2000} show the regret performance as a function of time for the two sub-networks of sizes $1000$ and $2000$ respectively. Note that the average regret is taken over 1000 experiments. For the $\epsilon_{t}$-greedy-LP policy, we let $c=5$ and $d=0.2$ (we choose $d=0.2$ to show that our algorithm seems to have good performance in more general settings, even when the bounds in the Proposition~\ref{prop:epsgreedySO} are not known or used). For both networks, we see that our policies outperform the UCB-N and UCB-MaxN policies~\cite{bhagat} (UCB-N and UCB-MaxN policies can also be viewed as special cases of those proposed in~\cite{chen2013combinatorial} for this specific combinatorial structure). We also observe that the improvement obtained by UCB-N policy over the baseline UCB1 policy is marginal. 
It has been shown~\citep{cooper} that for power law graphs, both $\gamma(G)$ and $\bar{\chi}(G)$ scale linearly with $N,$ although $\gamma(G)$ has a lower slope. Our numerical results show that our policies outperform existing policies even for the Flixster network. 

As we show in Corollary~\ref{prop:epsgreedyk} and Proposition~\ref{prop:ucblp}, $\epsilon_t$-greedy-LP and UCB-LP have the same upper bound $O(\sum_{j\in\cU}z^*_j\log T)$. It is hard to say which algorithm outperforms the other one. In the Flixster network, we see that the $\epsilon_{t}$-greedy-LP policy performs better than the UCB-LP policy. As we show in Section~\ref{sec:routing}, UCB-LP performs better than $\epsilon_t$-greedy-LP. In addition, the regret gap between the $\epsilon_t$-greedy-LP and UCB-LP is not large compared to their gain to UCB-N and UCB-maxN.


\subsection{Algorithm Performance on Synthetic Data}\label{sec:routing}
{We consider a routing problem defined on a communication network, which is demonstrated as an undirected graph consisting of $6$ nodes in Figure~\ref{fig:routingnetwork}. We assume that node $1$ is the source node and node $6$ is the destination node. The decision maker repeatedly sends packets from the source node to the destination node. There exist $13$ simple paths from the source node to the destination node. The delay of each path is the sum of delays over all the constituent links. The goal of the decision maker is to identify the path with the smallest expected delay and minimize the regret as much as possible. 

Solving the routing problem is a direct application of this work once we let the set of paths be the set of actions and the set of links be the set of base-arms. We assume that the delay of each link $i$, denoted by $X_i(t)$, is an independent and identically distributed sequence of random variables (drawn from the steady-state distribution) over discrete time $t$. Then, this is a stochastic bandit problem with side-observations since playing one action (path) reveals some observations of some base-arms (links) that contribute to other actions (paths). For example, choosing path $(1,2,4,6)$ reveals the delay performance of link $(1,2)$ and $(2,4)$, which are included in the path $(1,2,4,5,6)$. 

In the routing problem, there are $13$ paths and $12$ directed links (note that some links are never used in all the paths). Thus, we set $K=13$ and $N=12$ in the simulation. Then, we construct the set $V_j$ for each action $j$ such that $i\in V_j$ if path $j$ traverses the link $i$. And, we set $\cK_j=V_j$ for each $j\in\cK$ since the delay of a path is the total delay of the traversed links. Let $B$ be the upper bound of all the action delays. Then, we choose the function $f_j(\vec{X}_j(t))=1-\sum_{i\in\cK_j}X_{i}(t)/B$ as the reward of playing action $j$ at time $t$. In the simulation, we assume that the delay of link $i$ is sampled from a Bernoulli distribution with mean $u_i$. Each $u_i$ is independently sampled from a uniform distribution from $0$ to $1$, which realizes the problem instance in Figure~\ref{fig:routingnetwork}\footnote{The number indicates the mean delay and the arrow indicates the direction of the link}. One can check that the optimal action (shortest path) is the path $(1,3,5,6)$ given the ground truth $\{u_i\}_{i\in\cN}$. We let $B=5$ in the simulation.

 \begin{figure}[ht]
\centering
\includegraphics[width=3.0in]{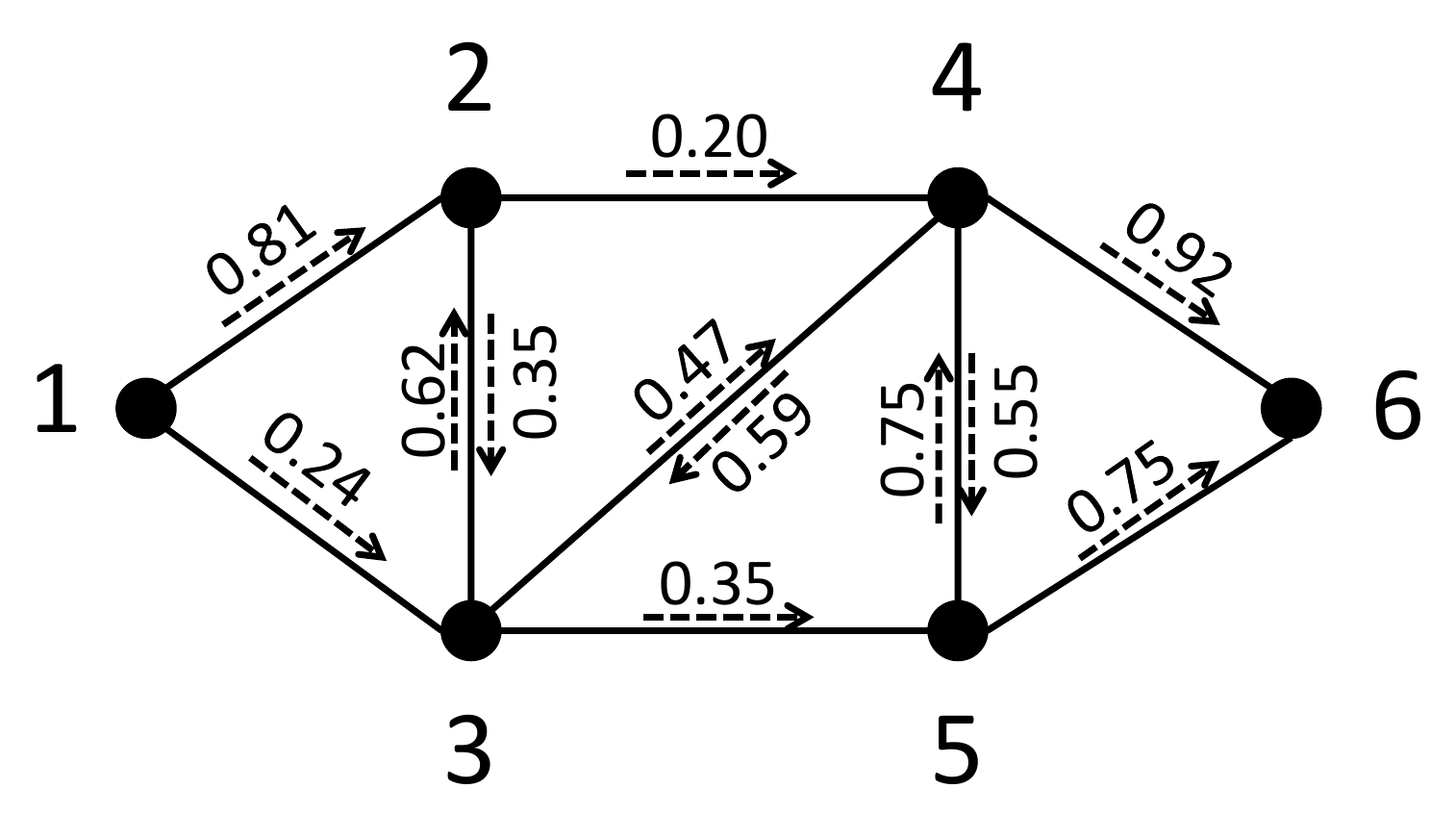}
\caption{Routing problem on a communication network}
\label{fig:routingnetwork}
\end{figure}
 \begin{figure}[ht]
\centering
\includegraphics[width=4.0in]{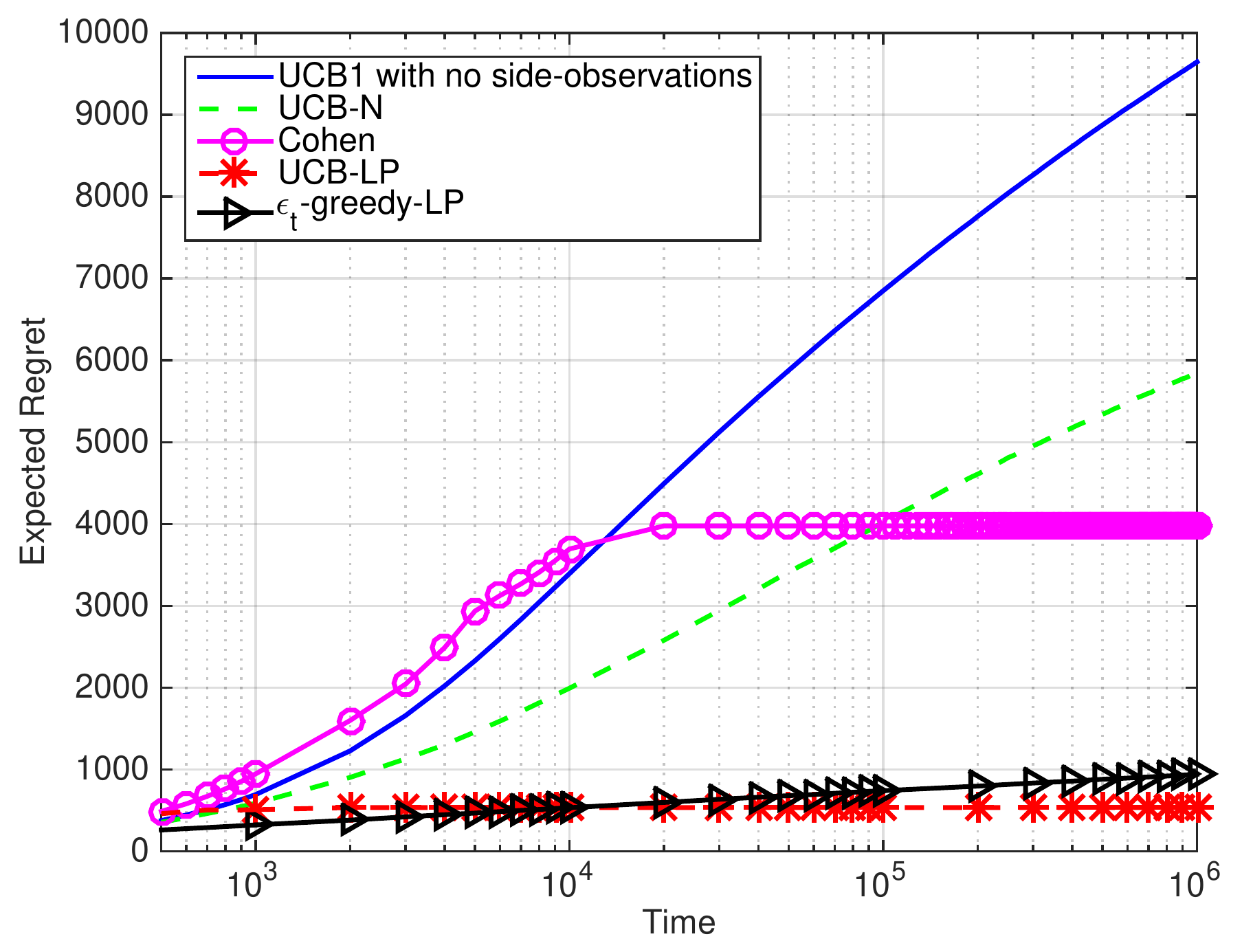}
\caption{Regret comparison of all the policies for the routing problem}
\label{fig:routing}
\end{figure}

We apply the UCB1, UCB-N, Cohen~\citep{Cohen2016} and our policies to the problem instance in Figure~\ref{fig:routingnetwork} and the regret performance, averaged over 1000 experiments, is shown in Figure~\ref{fig:routing}. We do not represent the result of the UCB-MaxN because it degenerates to the UCB-N policy. The reason is that there is no non-trivial clique (clique with more than one element) in this problem due to the setting $\cK_j=V_j$ and $\cK_j\neq\cK_a$ for any $j,a\in\cK$. Intuitively, there does not exist two paths that can observe the outcome of each other. For the $\epsilon_{t}$-greedy-LP policy, we let $c=4$ and $d=0.05.$ From the regret performance, we observe that the improvement obtained by the UCB-N policy against the UCB1 policy is considerably greater than the results in Figure~\ref{fig:1000} and Figure~\ref{fig:2000}. The reason behind this is that the bipartite graph in the routing problem is more dense and network size is small in the routing problem, which enables the UCB-N policy to take the advantage of side-observations. Overall, we see that our policies outperform the UCB-N policy and Cohen policy because our policies take the network structure into consideration, which enables us to trade off the exploration with exploitation more efficiently.} 

\subsection{Asymptotic Behavior}
We run a simulation to verify the result provided in Proposition~\ref{prop:lp2opt}. For each base-arm size $N$, we sequentially generate action $j$ such that $e_{ij}=1$ with probability $p$ for any $i\in \cN$ independently. Stopping time $\tau$ is the number of actions we have generated so that there are no useless base-arms in the network. Note that $\tau$ is an upper bound of $\gamma(G)$ and $\sum_{j\in\cK}z^*_j$ as shown in Appendix~\ref{app:lp2opt}. Then, we solve the linear program P2 to obtain $\sum_{j\in\cK}z^*_j$ and find a hitting set by a greedy algorithm.\footnote{It is well known that hitting set problem is NP-complete. So we employ the greedy algorithm which brings in the node with the largest degree in the network during each iteration.} 

\begin{figure}[ht]
\centering
\includegraphics[width=4.0in]{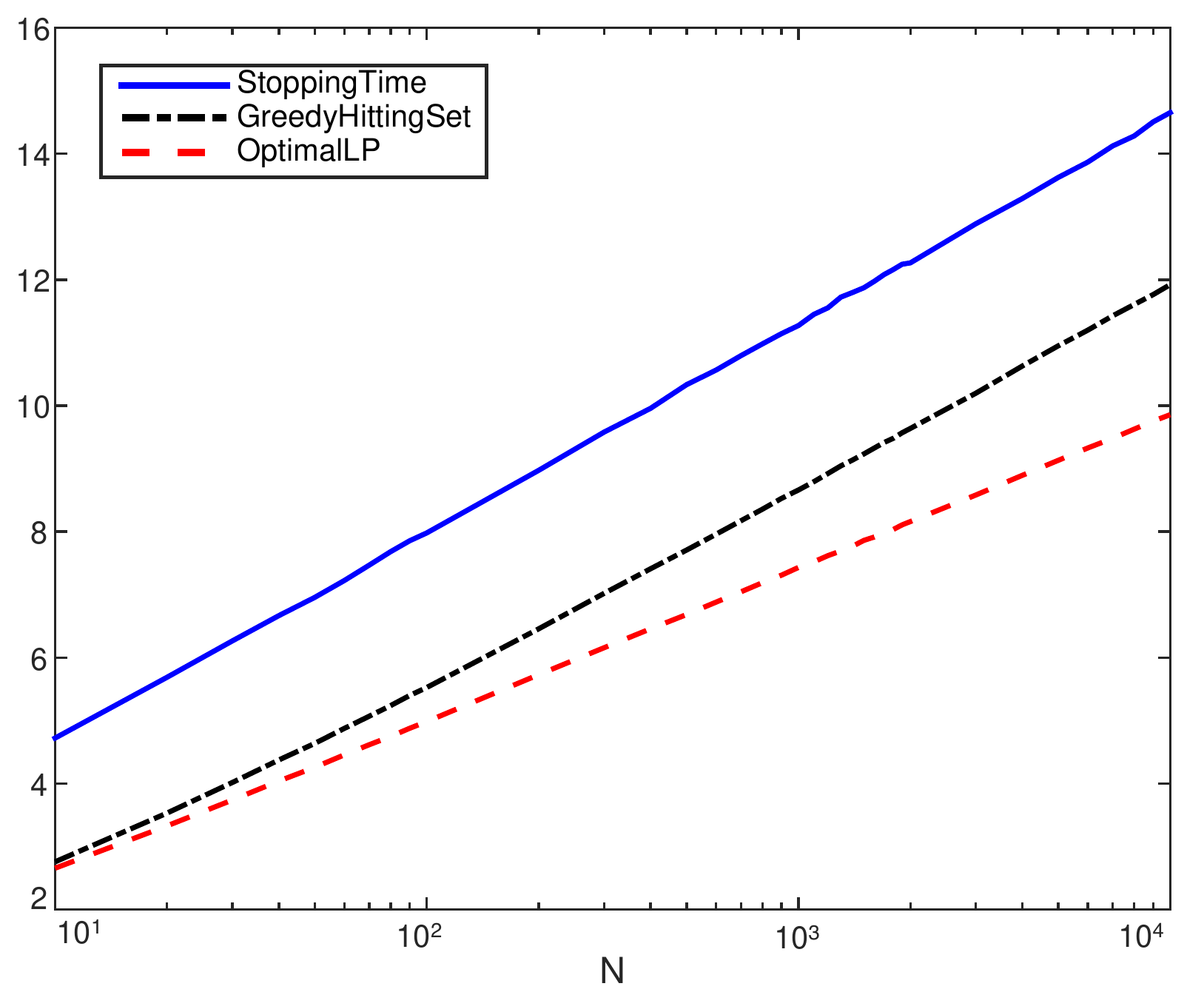}
\caption{Erdos-Renyi random graph with p=0.5}
\label{fig:stp}
\end{figure}

{Figure~\ref{fig:stp} shows the average result over 1000 samples for each $N$ when $p=0.5$. The numerical result verifies our theoretical result that $\sum_{j\in\cK}z^*_j$ is upper-bounded by a logarithmic function of $N$ asymptotically in Erdos-Renyi random graph. The reason why we are interested in the scaling order of $\sum_{j\in\cK}z^*_j$ is that the traditional UCB1 policy suffers from the curse of dimensionality when applied in the real world, such as recommendation systems with thousands of items. However, our policies show a regret of $O(\sum_{j\in\cK}z^*_j \log T)$, and $\sum_{j\in\cK}z^*_j$ is upper-bounded by a logarithmic function of the number of unknowns, which makes our policies scalable in some large networks.}

\section{Summary}\label{sec:conclude}
{In this work, we introduce an important structural form of feedback available in many multiarmed bandits using the bipartite network structure. We obtained an asymptotic (with respect to time) lower bound as a function of the network structure on the regret of any uniformly good policy. Further, we proposed two policies: 1) the $\epsilon_{t}$-greedy-LP policy, and 2) the UCB-LP policy, both of which are optimal in the sense that they achieve the asymptotic lower bound on the regret, up to a multiplicative constant that is independent of the network structure. These policies can have a better regret performance than existing policies for some important network structures. The $\epsilon_{t}$-greedy-LP policy is a network-aware any-time policy, but its exploration is oblivious to the average rewards of the suboptimal actions. On the other hand, UCB-LP considers both the network structure and the average rewards of actions. 

Important avenues of future work include the case of dynamic graphs -- what would be the lower bound and corresponding algorithms if the graph structure remains known but changes with time? Recently~\cite{aristide} presented a novel extension of Thompson sampling algorithm for the setting of immediate neighbor feedback studied in~\cite{mannor,bhagat,sigmetrics2014}. It would be interesting to see how to adapt Thompson sampling algorithm for the bipartite graph feedback structure introduced in our current work.}

\newpage

\appendix


In what follows, we give the proofs of all propositions stated in the earlier sections. These proofs make use of Lemmas~\ref{lem:chernoff},~\ref{lem:bernstein}, and~\ref{lem:prob}, and Proposition~\ref{prop:lailb} given in Section~\ref{sec:appendixSig}.
\section{Proof of Proposition~\ref{prop:lbSO}}\label{app:lbSO}

Let $\cU=\{j:\mu_{j}<\mu^{*}\}$ be the set of suboptimal actions. Also, let $\Delta_{j}=\mu^{*}-\mu_{j}.$  Also, $T_{j}(t)$ is the total number of times action $j$ is chosen up to time $t$ by policy $\bm{\phi}.$ 	Let $M_{i}(t)$ be the total number of observations corresponding to base-arm $i$ available at time $t.$
From Proposition~\ref{prop:lailb} given in the Appendix, we have, 
\begin{equation}\label{eq:lailb} \liminf_{t\rightarrow\infty}\frac{\mathbb{E}[M_{i}(t)]}{\log(t)}\ge\frac{1}{J_i(\theta_{i})},\  \forall i \in \cN. \end{equation}
	An observation is received for base-arm $i$ whenever any action in $S_{i}$ is chosen. 
	Hence, \begin{equation} \label{eq:sumT}M_{i}(t)=\sum_{j\in S_{i}}T_{j}(t).\end{equation} Now, from Equations~(\ref{eq:lailb}) and~(\ref{eq:sumT}), for each $i \in \cN,$
	\begin{equation} \label{eq:const}\liminf_{t\rightarrow\infty}\frac{\sum_{j\in S_{i}}\mathbb{E}[T_{j}(t)]}{\log(t)}\ge\frac{1}{J_i(\theta_{i})}.\end{equation} 
	Using~(\ref{eq:const}), we get the constraints of LP $P_{1}$. Further, we have from definition of regret that, $$\liminf_{t\rightarrow\infty}\frac{R_{\mu}(t)}{\log(t)}=\liminf_{t\rightarrow\infty}\sum_{j\in\cU}\Delta_{j}\frac{\mathbb{E}[T_{j}(t)]}{\log(t)}.$$ The above equation along with the constraints of the LP $P_{1}$ obtained from~(\ref{eq:const}) gives us the required lower bound on regret.

\section{Proof of Proposition~\ref{prop:lp2opt}}\label{app:lp2opt}

Here we consider a $E$-$R$ random graph with each entry of the matrix $E$ equals $1$ with probability $p$, where $0<p<1$.
Consider the following discrete stochastic process. $\chi_n$ are i.i.d., such that $\chi_n\subseteq[N]$ is sampled by the following steps: for each $i=1,2,..,N$, $i\in \chi_n$ with probability $p$. Let $q=1-p$. Then let $\tau$ be a stopping time defined as
\begin{equation}
\tau=\min\{n\geq1,\cup_{j=1}^{n}\chi_j=[N]\}
\end{equation}
The complement cdf of $\tau$ is the following.
\begin{equation}
P(\tau>n)=1-\left(1-q^n\right)^N
\end{equation}
\begin{enumerate}
\item Fix N. Given $0<p<1$, then $0<q<1$. Thus, $P(\tau=\infty)=0$
\item What is the upper bound of $E(\tau)$?
\begin{equation}
\left(1-q^n\right)^N>\exp(-\frac{q^nN}{1-q^n})~~(since~\ln(1-x)>-\frac{x}{1-x} ~~for ~~0<x<1)
\end{equation}
Thus, we have 
\begin{equation}
P(\tau>n)<1-\exp(-\frac{q^nN}{1-q^n})\leq\frac{q^nN}{1-q^n} ~~(since~1-e^x\leq-x)
\end{equation}
Then we can bound the expectation of $\tau$. 
\begin{equation}
E(\tau)=\sum_{n=1}^{\infty}P(\tau>n)<\sum_{n=1}^{\infty}q^n\frac{N}{1-q}=\frac{qN}{(1-q)^2}
\end{equation}
\item What is the upper bound of $P(\tau\leq n)$?
\begin{equation}
P(\tau\leq n)=\left(1-q^n\right)^N\leq \exp(-q^nN)
\end{equation}
\item Does $\tau$ converge as $N$ goes to infinity?\\
Given $\epsilon>0$, as $N$ goes to $\infty$,
\begin{align}
P(\tau\leq (1-\epsilon)\log_{1/q}N)&\leq\exp(-q^{(1-\epsilon)\log_{1/q}N}N)=\exp(-N^\epsilon)\rightarrow0\\
P(\tau>(1+\epsilon)\log_{1/q}N)&\leq\frac{q^{(1+\epsilon)\log_{1/q}N}N}{1-q}=\frac{1}{(1-q)N^\epsilon}\rightarrow0
\end{align}
Since $\epsilon$ is arbitrary, we can have
\begin{equation}
P(\tau=\log_{1/q}N)\rightarrow1 ~~as~ N\rightarrow\infty.
\end{equation}
That is to say $\tau$ converges to $\log_{1/q}N$ in probability.
\end{enumerate}

Suppose there are no useless base-arms in the network, i.e. $[K]$ is a hitting set. Then $\tau$ is less than $K$ with probability 1. Given this information, $\gamma(G)$ should be upper bounded by $\log_{1/q}N$ as $N$ goes to infinity.
	
\section{Proof of Proposition~\ref{prop:lbk}}\label{app:lbk}

Let $(z_{j}^{*})_{j\in\cK}$ be the optimal solution of LP $P_{2}.$ We will first prove the upper bound in Equation~\ref{eq:lb}. 
Using the optimal solution $(w_{j}^{*})_{j\in\cK}$ of LP $P_{1},$ we construct a feasible solution satisfying constraints in LP $P_{2}$ in the following way: For actions $j \in \cK,$ let $\displaystyle z_{j} = \left(\max_{i\in\cN} J_i(\theta_i) \right) w_{j}^{*}.$ Then $(z_{j})_{j\in\cK}$ satisfy constraints for all base-arms $i \in \cN$
because $w_{j}^{*}$ satisfy constraints of LP $P_{1}.$ \\
The feasible solution constructed in this way gives an upper bound on the optimal value of LP $P_{2}.$ Hence, 
\begin{align*}
\sum_{j \in \cK}z_{j}^{*} &\le \sum_{j \in \cU} z_{j} + |\cO|\\
&\le \sum_{j \in \cU}\left(\max_{i\in\cN} J_i(\theta_i)\right) w_{j}^{*} + |\cO| \\
&\le  \frac{\max_{i\in\cN} J_i(\theta_i)}{\min_{j\in\cU}\Delta_j}\sum_{j \in \cU}\Delta_{j}w_{j}^{*}+|\cO| \\
&\le \frac{\max_{i\in\cN} J_i(\theta_i)}{\min_{j\in\cU}\Delta_j}c_{\bm{\mu}}+|\cO|
\end{align*}
For the lower bound, any feasible solution of $P_{2},$ in particular $\mathbf{z}^{*},$ can be used to construct a feasible solution of $P_{1}.$ For actions $j \in \cK,$ let $\displaystyle w_{j} = \frac{z_{j}^{*}}{\min_{i\in\cN} J_i(\theta_i)}.$ Then $(w_{j})_{j\in\cK}$ satisfies the constraints of LP $P_{1}$ and hence gives an upper bound on its optimal value. Therefore, we have 
\begin{align*}
c_{\bm{\mu}} &= \sum_{j \in \cU}\Delta_{j} w_{j}^{*} ,\\ &\le \sum_{j \in \cK}\frac{\Delta_j z_{j}^{*}}{\min_{i\in\cN} J_i(\theta_i)} \\ &\le \sum_{j \in \cK}\frac{\max_{a\in\cU}\Delta_a z_{j}^{*}}{\min_{i\in\cN} J_i(\theta_i)} 
\end{align*}
which gives us the required lower bound. 

\section{Proof of Proposition~\ref{prop:epsgreedySO}}\label{app:epsgreedySO}
Since $\mathbf{z}^{*}$ satisfies the constraints in LP $P_{2},$ there is sufficient exploration within each suboptimal action's neighborhood. The proof is then a combination of this fact and the proof of Theorem~3 in ~\cite{AuerFinite}. Let $\bar{f}_j(t)$ be the random variable denoting the sample mean of all observations available for action $j$ at time $t.$ Let $\bar{f}^*(t)$ be the random variable denoting the sample mean of all observations available for an optimal action at time $t.$ Fix a suboptimal action $j.$ 
For some $\alpha>1,$ define $m_i$ for each base-arm $i$ as follows, $$m_i=\frac{1}{\alpha}\frac{\sum_{j\in S_i}z_j^*}{\sum_{j\in\cK}z_j^*}\sum_{m=1}^{t}\epsilon(m)$$\\
Let $\phi(t)$ be the action chosen by $\epsilon_t$-greedy-LP policy at time $t.$ The event $\{\phi(t)=j\}$ implies that either sampling a random action $j$ for exploration or playing the best observed action $j$ for exploitation. Then, 

$$\bP[\phi(t)=j]\le \frac{\epsilon(t)z_j^*}{\sum_{a\in\cK}z_a^*}+(1-\epsilon(t))\bP[\bar{f}_j(t)\ge\bar{f}^*(t)]$$
The event $\{\bar{f}_j(t)\ge\bar{f}^*(t)\}$ implies that either $\{\bar{f}_j(t)\ge\mu_j+\frac{\Delta_j}{2}\}$ or $\{\bar{f}^*(t)\le\mu^*-\frac{\Delta_j}{2}\}$ since $\mu_j+\frac{\Delta_j}{2}=\mu^*-\frac{\Delta_j}{2}$. We also have that,
\begin{align*}\bP[\bar{f}_j(t)\ge\bar{f}^*(t)]&\le \bP\left[\bar{f}_j(t)\ge\mu_j+\frac{\Delta_j}{2}\right] +\bP\left[\bar{f}^*(t)\le\mu^*-\frac{\Delta_j}{2}\right].\end{align*}
The analysis of both the terms in the right hand side of the above expression is similar. Let $O_i^{R}(t)$ be the total number of observations available for base-arm $i$ from the exploration iterations of the policy up to time $t.$ Let $O_i(t)$ be the total number of observations available for base-arm $i$ up to time $t.$ By concentration inequalities, the probability that the empirical mean deviate from the expectation can be bounded given the number of observations. The number of observations for action $j$ is lower-bounded by the number of observations from the exploration iterations. Hence, we have,  
\begin{align*}
\bP\left[\bar{f}_j(t)\ge\mu_j+\frac{\Delta_j}{2}\right]&=\sum_{m=1}^t\bP\left[\min_{i\in\mathcal{K}_j}O_i(t) = m;\bar{f}_j(t)\ge\mu_j+\frac{\Delta_j}{2}\right]\\
&=\sum_{m=1}^t\bP\left[\bar{f_j}(t)\geq\mu_j+\frac{\Delta_j}{2}|\min_{i\in\mathcal{K}_j}O_i(t) = m\right]\bP\left[\min_{i\in\mathcal{K}_j}O_i(t) = m\right]\\
&\le\sum_{m=1}^{t}\bP\left[\min_{i\in\mathcal{K}_j}O_i(t) = m\right]e^{\frac{-\Delta_j^2 m}{2}}\\
&\mbox{(follows from Chernoff-Hoeffding bound in Lemma~\ref{lem:chernoff})}\\
&\le\sum_{m=1}^{t}\bP\left[\min_{i\in\mathcal{K}_j}O_i^R(t) \leq m\right]e^{\frac{-\Delta_j^2 m}{2}}\\
&\leq \sum_{m=1}^{\lfloor m_0 \rfloor}\bP\left[\min_{i\in\mathcal{K}_j}O_i^R(t) \leq m\right] + \sum_{m=\lfloor m_0 \rfloor +1}^{t}e^{\frac{-\Delta_j^2m}{2}}\\
&\leq m_0\bP\left[\min_{i\in\mathcal{K}_j}O_i^R(t) \leq m_0\right] + \frac{2}{\Delta_j^2}e^{\frac{-\Delta_j^2m_0}{2}}\\
&\qquad \left(\mbox{since } \sum_{m+1}^{\infty}e^{-k u} \leq \frac{1}{k}e^{-km} \right) \\
&\leq \sum_{i\in \mathcal{K}_j}m_0\bP\left[O_i^R(t) \leq m_0\right] + \frac{2}{\Delta_j^2}e^{\frac{-\Delta_j^2m_0}{2}},
\end{align*}
where $m_0=\min_{i\in \mathcal{N}}m_i$.

Recall that $O_i^R(t)$ is the total number of observations for base-arm $i$ from exploration. Now, we derive the bounds for the expectation and variance of $O_i^R(t)$ in order to use Bernstein's inequality. 
\begin{align*}
\mathbb{E}\left[O_i^R(t)\right]&=\sum_{m=1}^{t}\epsilon(m)\sum_{j\in S_i}\frac{z_j^*}{\sum_{j\in \mathcal{K}}z_j^*}\\
&=\frac{\sum_{j\in S_i}z_j^*}{\sum_{j\in \mathcal{K}}z_j^*}\sum_{m=1}^{t}\epsilon(m)=\alpha m_i\\
&\geq \alpha m_0
\end{align*}
\begin{align*}
var\left[O_i^R(t)\right]&=\sum_{m=1}^{t}\left(\epsilon(m)\sum_{j\in S_i}\frac{z_j^*}{\sum_{j\in \mathcal{K}}z_j^*}\right)\left(1-\epsilon(m)\sum_{j\in S_i}\frac{z_j^*}{\sum_{j\in \mathcal{K}}z_j^*}\right)\\
&\leq \sum_{m=1}^{t}\epsilon(m)\sum_{j\in S_i}\frac{z_j^*}{\sum_{j\in \mathcal{K}}z_j^*}\\
&=\mathbb{E}[O_i^R(t)]=\alpha m_i
\end{align*}

Now, using Bernstein's inequality given in Lemma~\ref{lem:bernstein}, we have 
\begin{align*}
\bP\left[O_i^R(t)\leq m_0\right] &= \bP\left[O_i^R(t)\leq\mathbb{E}[O_i^R(t)]+m_0-\alpha m_i\right]\\
&\leq \bP\left[O_i^R(t)\leq\mathbb{E}[O_i^R(t)]+m_i-\alpha m_i\right]\\
&\leq \exp\left(-\frac{(\alpha -1)^2m_i^2}{2\alpha m_i + \frac{2}{3}(\alpha - 1)m_i}\right)\\
&=\exp \left(-\frac{3(\alpha-1)^2}{8\alpha-2}m_i\right)=\exp(-rm_i)\\
\end{align*}
where $r=\frac{3(\alpha-1)^{2}}{8\alpha-2}.$
Now, we will obtain upper and lower bounds on $m_i$ by plugging in the definition of $\epsilon(m)$. For the upper bound, for any $t>t'=\frac{c\sum_{i\in\cK}z_i^*}{d^2},$
\begin{align*}
m_i&=\frac{\sum_{j\in S_i}z_j^*}{\alpha \sum_{j\in \mathcal{K}}z_j^*}\sum_{m=1}^{t}\epsilon(m)\\
&=\frac{\sum_{j\in S_i}z_j^*}{\alpha \sum_{j\in \mathcal{K}}z_j^*}t^\prime + \frac{\sum_{j\in S_i}z_j^*}{\alpha \sum_{j\in \mathcal{K}}z_j^*}\sum_{m=t^\prime+1}^{t}\frac{c\sum_{i\in\mathcal{K}}z_i^*}{d^2m}\\
&\leq \frac{c\delta}{\alpha d^2}\left(1+\sum_{m=t^\prime+1}^{t}\frac{1}{m}\right)\\
&\leq \frac{c\delta}{\alpha d^2}\log\left(\frac{e^2t}{t^\prime}\right).
\end{align*}
where $\delta=\max_{i\in\mathcal{N}}|S_i|$, denoting the maximum degree of the supports in the network. In the above, $\sum_{j\in S_i}z_j^* \le \delta$ because $z_j^*\le 1,$ which is due to the fact that $\left(z_j^*\right)_{j\in\cK}$ is the optimal solution of LP $P_2.$  Next, for the lower bound, we use the fact that $\sum_{j\in S_i}z_j^* \ge 1$ for all $i$ because $\left(z_j^*\right)_{j\in\cK}$ satisfies the constraints of LP $P_2.$ Thus
\begin{align*}
m_i&\geq\frac{\sum_{j\in S_i}z_j^*}{\alpha \sum_{j\in \mathcal{K}}z_j^*}\sum_{m=t^\prime+1}^{t}\frac{c\sum_{i\in\mathcal{K}}z_i^*}{d^2m}\\
&\geq \frac{c}{\alpha d^2}\sum_{m=t^\prime+1}^{t}\frac{1}{m}\\
&\geq \frac{c}{\alpha d^2}\log\left(\frac{t}{et^\prime}\right).
\end{align*}
Let $\lambda=\max_{j\in \mathcal{K}}|\mathcal{K}_j|$.
Hence, combining the inequalities above,
\begin{align*}
\bP\left[\bar{f_j}(t)\geq \mu_j+\frac{\Delta_j}{2}\right]&\leq \sum_{i\in \mathcal{K}_j}m_0\bP\left[O_i^R(t)\leq m_0\right]+\frac{2}{\Delta_j^2}e^{-\frac{\Delta_j^2m_0}{2}}\\
&\leq  \sum_{i\in \mathcal{K}_j}m_0 \left(\frac{et^\prime}{t}\right)^{cr/\alpha d^2} +\frac{2}{\Delta_j^2}e^{-\frac{\Delta_j^2m_0}{2}}\\
&\leq \lambda \frac{c\delta}{\alpha d^2}\left(\log \left(\frac{e^2t}{t^\prime}\right)\right)\left(\frac{et^\prime}{t}\right)^{cr/\alpha d^2} +\frac{2}{\Delta_j^2}\left(\frac{et^\prime}{t}\right)^{\frac{c\Delta_j^2}{2\alpha d^2}}
\end{align*}
Now, similarly for the optimal action, we have, for all $t>t'$ 
\begin{align*}
\bP\left[\bar{f^*}(t)\leq\mu^*-\frac{\Delta_j}{2}\right]\leq\frac{\lambda c \delta}{\alpha d^2}\left(\frac{et^\prime}{t}\right)^{cr/\alpha d^2} \log \left(\frac{e^2t}{t^\prime}\right)+\frac{2}{\Delta_j^2}\left(\frac{et^\prime}{t}\right)^{\frac{c\Delta_j^2}{2\alpha d^2}}.
\end{align*}
Combining everything, we have for any suboptimal action $j,$ for all $t>t'$
\begin{align*}
\bP[\phi(t)=j]&\leq \frac{\epsilon(t)z_j^*}{\sum_{a\in \mathcal{K}}z_a^*} +\left(1-\epsilon(t)\right)P[\bar{f_j}(t)\geq\bar{f^*}(t)]\\
&\leq \frac{cz_j^*}{d^2t}+P[\bar{f_j}(t)\geq \bar{f^*}(t)]\\
&\leq \frac{cz_j^*}{d^2t}+\frac{2\lambda c \delta}{\alpha d^2}\left(\frac{et^\prime}{t}\right)^{cr/\alpha d^2} \log \left(\frac{e^2t}{t^\prime}\right)+\frac{4}{\Delta_j^2}\left(\frac{et^\prime}{t}\right)^{\frac{c\Delta_j^2}{2\alpha d^2}}
\end{align*}

\section{Proof of Proposition~\ref{prop:ucblp}}\label{app:ucblpSO}
The proof technique is similar to that in~\citet{AuerRevised}. We will analyze the regret by conditioning on two disjoint events. The first event is that each suboptimal action $a$ is eliminated by an optimal action on or before the first round $m$ such that $\tilde{\Delta}_{m}<\Delta_{a}/2.$ This happens with high probability and leads to logarithmic regret. The compliment of the first event yields linear regret in time but occurs with probability proportional to $1/T.$ The main difference from the proof in~\citet{AuerRevised} is that on the first event, the number of times we choose each action $j$ is proportional to $z_{j}^{*}\log(T)$ in the exploration iterations (i.e., when $|B_m|>1$) of the policy. This gives us the required upper bound in terms of optimal solution $\mathbf{z}^{*}$ of LP $P_{2}.$

Let $*$ denote any optimal action. Let $m^*$ denote the round in which the last optimal action $*$ is eliminated. For each suboptimal action $j$, define round $m_j:=\min\{m : \tilde{\Delta}_{m}<\frac{\Delta_j}{2}\}.$ For an optimal action $j,$ $m_{j}=\infty$ by convention.
Then, by the definition of $m_j,$ for all rounds $m<m_j,$ $\Delta_j \le 2\tDelta_m,$ and \begin{equation}\label{eq:mirel}\frac{2}{\Delta_j}< 2^{m_j}=\frac{1}{\tDelta_{m_j}}\le\frac{4}{\Delta_j}<\frac{1}{\tDelta_{m_j+1}}=2^{m_j+1}.\end{equation}
From Lemma~\ref{lem:prob} in the Appendix, the probability that action $j$ is not eliminated in round $m_j$ by $*$ is at most $\frac{2}{T\tDelta_{m_j}^2}.$ \\ 
Let $I(t)$ be the action chosen at time $t$ by the UCB-LP policy. \\ 
Let $E_{m^*}$ be the event that all suboptimal actions with $m_j\le m^*$ are eliminated by $*$ on or before their respective $m_j$. Then, the complement of $E_{m^*},$ denoted as $E_{m^*}^c,$ is the event that there exists some suboptimal action $j$ with $m_j\le m^*,$ which is not eliminated by round $m_j.$ Let $E_j^c$ be the event that action $j$ is not eliminated by round $m_j$ by $*.$ Let $m_f=\lfloor\frac{1}{2}\log_2\frac{T}{e}\rfloor$ and $I(t)$ denote the action chosen at time $t$ by the policy.  Recall that regret is denoted by $R_{\bm{\mu}}(T).$ Let $\bP[m^*=m]$ be denoted by $p_{m}.$ Hence, $\sum_{m=0}^{m_f}p_{m}=1.$
\begin{align*}
\bE\left[R_{\bm{\mu}}(T)\right] & = \sum_{m=0}^{m_f} \bE\left[R_{\bm{\mu}}(T)|\{m^*=m\}\right]\bP[m^*=m] \\& =\sum_{m=0}^{m_f} \sum_{t=1}^T \sum_{j\in\cU}\Delta_j \bP\left[I(t)=j|\{m^*=m\}\right]p_{m}\\
&=\sum_{m=0}^{m_f} \sum_{t=1}^T \sum_{j\in \cU}\Delta_j \bP\left[\{I(t)=j\}\cap E_{m^*} |\{m^*=m\}\right]p_{m}\\& \ +\sum_{m=0}^{m_f} \sum_{t=1}^T \sum_{j\in \cU}\Delta_j\bP\left[\{I(t)=j\}\cap E_{m^*}^c |\{m^*=m\}\right]p_{m}\\
&= (\romannumeral 1) + (\romannumeral 2)
\end{align*}
Next we will show that term $(\romannumeral 1)$ leads to logarithmic regret while term $(\romannumeral 2)$ leads to a constant regret with time.\\
First, consider the term $(\romannumeral 2)$ of the regret expression. 
For each $j\in\cU,$ we have,
\begin{align*}
\sum_{m=0}^{m_f} \sum_{t=1}^T & \bP\left[\{I(t)=j\}\cap E_{m^*}^c |\{m^*=m\}\right]\bP[m^*=m]\\
&\le\sum_{m=0}^{m_f} \sum_{t=1}^T  \bP\left[\{I(t)=j\}\cap\left( \cup_{a\in\cU:m_a \le m^*} E_a^c\right) |\{m^*=m\}\right]p_{m}\\
&\le \sum_{m=0}^{m_f} \sum_{t=1}^T  \Big(\bP\left[\{I(t)=j\}|\left( \cup_{a\in\cU:m_a \le m^*} E_a^c\right), \{m^*=m\}\right]\\&\qquad\qquad\qquad\qquad\bP\left[ \cup_{a\in\cU:m_a \le m^*} E_a^c|\{m^*=m\}\right]p_{m}\Big)\\
&\le T\bP\left[ \cup_{a\in\cU} E_a^c|\{m^*=m_{f}\}\right]\sum_{m=0}^{m_f}p_{m}
\\&\le T\sum_{a\in\cU}\frac{2}{T\tDelta_{m_a}^2}, \\ &\quad \left(\mbox{using~Lemma~\ref{lem:prob}, } \bP\left[E_a^c|\{m^*=m_{f}\}\right] \le\frac{2}{T\tDelta_{m_a}^2}\right)\\&\le \sum_{a\in\cU}\frac{32}{\Delta_{a}^2},
\end{align*}
where the last inequality follows from Equation~(\ref{eq:mirel}).
Hence, the term $(\romannumeral 2)$ of regret is 
\begin{align}\label{eq:2}\sum_{m=0}^{m_f} \sum_{t=1}^T \sum_{j\in \cU}\Delta_j &\bP\left[\{I(t)=j\}\cap E_{m^*}^c |\{m^*=m\}\right]p_{m}\nonumber\\&\qquad \le  \sum_{j\in \cU}\Delta_j \sum_{a\in\cU}\frac{32}{\Delta_{a}^2}=O(K^{2}).\end{align}
\\ \\
Next, we consider the term $(\romannumeral 1)$. Recall that, in this term, we consider the case that all suboptimal actions $j$ with $m_{j}\le m^{*}$ are eliminated by $*$ on or before $m_j.$  
\begin{align*}
(\romannumeral 1)&=\sum_{m=0}^{m_f}  \sum_{t=1}^T  \sum_{j\in \cU}\Delta_j \bP\left[\{I(t)=j\}\cap E_{m^*} |\{m^*=m\}\right]p_m\\
&=\sum_{m=0}^{m_f}\bE\left[R_{\bm{\mu}}(T)|\{m^*=m\},E_{m^*}\right]\bP[E_{m^*}|\{m^*=m\}]p_{m}\\
&\le \sum_{m=0}^{m_f} \Big(\bE\left[\mbox{Regret from $\{j: m_j\le m^*\}$}|\{m^*=m\},E_{m^*}\right]\\&\ \quad+\bE\left[\mbox{Regret from $\{j: m_j > m^*\}$}|\{m^*=m\},E_{m^*}\right]\Big)p_{m}\\
&\le \sum_{m=0}^{m_f} \Big(\bE\left[\mbox{Regret from $\{j: m_j\le m_{f}\}$}|\{m^*=m_{f}\},E_{m_{f}}\right]\\&\ \quad+\bE\left[\mbox{Regret from $\{j: m_j > m^*\}$}|\{m^*=m\},E_{m^*}\right]\Big)p_{m}\\ 
&\le \bE\left[R_{\bm{\mu}}(T)|\{m^{*}=m_f\},E_{m_f}\right]\sum_{m=0}^{m_f}p_{m} \\& +\sum_{m=0}^{m_{f}}\bE\left[\mbox{Regret from $\{j:m_j > m^*\}$}|\{m^*=m\},E_{m^*}\right]p_{m}\\
&=(\romannumeral 1a) + (\romannumeral 1b)
\end{align*}
Once again, we will consider the above two terms separately. For the term $(\romannumeral 1a),$ under the event $E_{m_f},$ each suboptimal action $j$ is eliminated by $*$ by round $m_{j}.$ Define round $\bar{m}$ and the set $B$ as follows: 
\begin{align*}\bar{m}&=\min\{m: \sum_{j\in \cK}z_{j}^{*} > \sum_{a:m_{a}>m}2^{-m+1}\},\\ B&=\{j \in \cU:m_{j}>\bar{m}\}.\end{align*} 
After round $\bar{m}$, Algorithm~\ref{alg:ucblp} chooses only those actions with $m_{j} > \bar{m}.$ Also, by the definition of the {\bf Reset} phase of Algorithm~\ref{alg:ucblp}, we have that any suboptimal action $j \notin B$ is chosen (i.e. appears in the set $A_{m}$ at round $m$) only until it is not in $A_m$ or until $\bar{m},$ whichever happens first. Define $\displaystyle n_{j} = \min\{\bar{m}, \max_{a:j\in G_a}\{m_{a}\}\}$ for each suboptimal action $j$, where $G_a=\bigcup_{i\in \cK_a} S_i$ for action $a$. Then any suboptimal action $j \notin B$ is chosen for at most $n_{j}$ rounds. 
\begin{align}\label{eq:1a}
(\romannumeral 1a)&=\bE\left[R_{\bm{\mu}}(T)|\{m^{*}=m_f\},E_{m_f}\right] \nonumber \\ &\le \sum_{j\in \cU\setminus B} \Delta_{j} z_{j}^{*}\frac{2\log(T\tDelta_{n_{j}}^{2})}{\tDelta_{n_{j}}^{2}} + \sum_{j\in B} \Delta_{j}\frac{2\log(T\tDelta_{m_{j}}^{2})}{\tDelta_{m_{j}}^{2}}\nonumber \\
&\le \sum_{j\in \cU\setminus B} \Delta_{j} z_{j}^{*}\frac{32\log(T\hat{\Delta}_{j}^{2})}{\hat{\Delta}_{j}^{2}} + \sum_{j\in B} \Delta_{j}\frac{32\log(T\Delta_{j}^{2})}{\Delta_{j}^{2}}, 
\end{align}
where $\hat{\Delta}_{j}=\max\{2^{-\bar{m}+2},\min_{a:j\in G_a}\{\Delta_{a}\}\}$ and $(z_{j}^{*})$ is the solution of LP $P_{2}.$\\
Finally, we consider the term $(\romannumeral 1b).$ Note that $T_j(m)\geq n(m), \forall j \in B_m, \forall m.$ An optimal action $*$ is not eliminated in round $m^{*}$ if~(\ref{eq:sub1}) holds for $m=m^*$. Hence, using~(\ref{eq:cbl}) and~(\ref{eq:cbu}), the probability $p_{m}$ that $*$ is eliminated by a suboptimal action in any round $m^{*}$ is at most $\frac{2}{T\tDelta_{m^{*}}^{2}}.$ Hence, term $(\romannumeral 1b)$ is given as:
\begin{align}
\sum_{m=0}^{m_{f}}&\bE\left[\mbox{Regret from $\{j:m_j > m^*\}$}|\{m^*=m\},E_{m^*}\right]p_{m}\nonumber \\
&\le \sum_{m=0}^{m_{f}} \sum_{j\in\cU:m_{j}\ge m}\frac{2}{T\tDelta_{m}^{2}}.T\max_{a\in\cU}\Delta_{a}\nonumber\\
&\le \max_{a\in\cU}\Delta_{a}\sum_{m=0}^{m_{f}} \sum_{j\in\cU:m_{j}\ge m}\frac{2}{\tDelta_{m}^{2}}\nonumber \\ \label{eq:1b}
&\le  \sum_{j\in\cU}\sum_{m=0}^{m_{j}}\frac{2}{\tDelta_{m}^{2}}\nonumber\\
&\le \sum_{j\in\cU}2^{2m_{j}+2}\le \sum_{j\in\cU}\frac{64}{\Delta_{j}^{2}}=O(K).
\end{align}

\noindent Now we get the result~(\ref{eq:ucblp}) by combining the bounds in~(\ref{eq:2}), (\ref{eq:1a}), and~(\ref{eq:1b}). \\
Further, the definition of set $B$ ensures that we have $$\sum_{j\in B}\Delta_{j}\le \sum_{j\in\cK}z_{j}^{*}.$$ Also, using the Assumption~\ref{assum:bounded}, $\frac{32\Delta_{j}\log(T\hat{\Delta}_{j}^{2})}{\hat{\Delta}_{j}^{2}}, \frac{32\log(T\Delta_{j}^{2})}{\Delta_{j}^{2}}$ are bounded by $C\log(T),$ where $C=\frac{32}{\min_{j\in\cU}\Delta_j^2}$, is a constant independent of  network structure. When one checks the feasibility of $C$, note that $\hat{\Delta}_j\geq \min_{a:j\in G_a}\Delta_a$ by definition and $\Delta_j\leq 1$ for any $j$ since the rewards are bounded by $1$. Hence,~(\ref{eq:1a}) can be bounded as:  
\begin{align}
\sum_{j\in \cU\setminus B}& \Delta_{j} z_{j}^{*}\frac{32\log(T\hat{\Delta}_{j}^{2})}{\hat{\Delta}_{j}^{2}} + \sum_{j\in B} \Delta_{j}\frac{32\log(T\Delta_{j}^{2})}{\Delta_{j}^{2}} \nonumber \\ &\qquad \le \sum_{j\in \cU\setminus B} z_{j}^{*}C\log(T) + \sum_{j\in B} \Delta_{j}C\log(T) \nonumber \\ 
&\qquad \le \sum_{j\in \cU\setminus B} z_{j}^{*}C\log(T) + \sum_{j\in B} 2^{-\bar{m}+1}C\log(T) \nonumber \\ 
&\qquad\le 2\sum_{j\in \cK} z_{j}^{*}C\log(T). \label{eq:temp1}
\end{align}
Hence, we get~(\ref{eq:ucblp2}) from~(\ref{eq:temp1}),~(\ref{eq:2}), and~(\ref{eq:1b}).

\section{Supplementary Material} \label{sec:appendixSig}
$S_{n}=\frac{1}{n}\sum_{j=1}^{n}X_{j}$ denotes the sample mean of the random variables $X_{1},\ldots,X_{n}.$
The first two lemmas below state the Chernoff-Hoeffding inequality and Bernstein's inequality.
\begin{lemma}\label{lem:chernoff}
Let $X_{1},\ldots,X_{n}$ be a sequence of random variables with support $[0,1]$ and $\mathbb{E}[X_{t}]=\mu$ for all $t\le n.$ Let $S_{n}=\frac{1}{n}\sum_{j=1}^{n}X_{j}.$ Then, for all $\epsilon>0,$ we have,
\begin{align*}
\mathbb{P}[S_{n}\ge\mu+\epsilon]&\le e^{-2n\epsilon^{2}}\\
\mathbb{P}[S_{n}\le\mu-\epsilon]&\le e^{-2n\epsilon^{2}}.
\end{align*}
\end{lemma}
\begin{lemma}\label{lem:bernstein}
Let $X_{1},\ldots,X_{n}$ be a sequence of random variables with support $[0,1]$ and $\sum_{k=1}^{t}$ $var[X_{k}|X_{1},\ldots,X_{k-1}]\le\sigma^{2}$  for all $t\le n.$ Let $S_{n}=\sum_{j=1}^{n}X_{j}.$ Then, for all $\epsilon>0,$ we have,
\begin{align*}
\mathbb{P}[S_{n}\ge\mathbb{E}[S_{n}]+\epsilon]&\le \exp\left\{-\frac{\epsilon^{2}}{2\sigma^{2}+\frac{2}{3}\epsilon}\right\}\\
\mathbb{P}[S_{n}\le\mathbb{E}[S_{n}]-\epsilon]&\le \exp\left\{-\frac{\epsilon^{2}}{2\sigma^{2}+\frac{2}{3}\epsilon}\right\}.
\end{align*}
\end{lemma}
The next lemma is used in the proof of Proposition~\ref{prop:ucblp}.
\begin{lemma}
\label{lem:prob}The probability that action $j$ is not eliminated in round $m_j$ by $*$ is at most $\frac{2}{T\tDelta_{m_j}^2}.$
\end{lemma}
\begin{proof}
Let $\bar{f}_{j}(m)$ be the sample mean of all observations for action $j$ available in round $m.$ Let $\bar{f}^{*}(m)$ be the sample mean of the optimal action. The constraints of LP $P_2$ ensure that at the end of each round $m,$ for all  actions in $B_m,$ we have at least $n(m):=\left\lceil \frac{2\log(T\tilde{\Delta}_{m}^{2})}{\tilde{\Delta}_{m}^{2}}\right\rceil$ observations. The reason is as follows. The set $A_{m}$ contains set $B_{m}$. In particular, $A_m=\cup_{i\in D_m}S_i$ and $D_m=\cup_{j\in B_m}\cK_j$. If each action $j$ in $A_m$ is played $z^*_j$ times, then all the base-arms in $D_m$ have at least 1 observations according the constraints of LP $P_2$. Thus, the actions in $B_m$ have at least 1 observations. In sum, for all actions in $B_m$, we have at least $n(m)-n(m-1)$ observations at round m. Thus, we have at least $n(m)$ observations for all actions in $B_m$.

Now, for $m=m_j,$ if we have,
\begin{equation}\label{eq:sub1}\bar{f}_{j}(m)\le \mu_j+\sqrt{\frac{\log(T\tilde{\Delta}_{m}^{2})}{2n(m)}} \ \mbox{ and } \ \bar{f}^*(m)\ge \mu^*-\sqrt{\frac{\log(T\tilde{\Delta}_{m}^{2})}{2n(m)}},\end{equation} 
then, action $j$ is eliminated by $*$ in round $m_j.$ In fact, in round $m_j,$ we have $$\sqrt{\frac{\log(T\tDelta_{m_j}^2)}{2n(m_j)}}\le \frac{\tDelta_{m_j}}{2}<\frac{\Delta_j}{4}.$$ Hence, in the elimination phase of the UCB-LP policy, if~(\ref{eq:sub1}) holds for action $j$ in round $m_j,$ we have,
\begin{align*}
\bar{f}_{j}(m_j)+\sqrt{\frac{\log(T\tilde{\Delta}_{m_j}^{2})}{2n(m_j)}} &\le \mu_{j}+2\sqrt{\frac{\log(T\tilde{\Delta}_{m_j}^{2})}{2n(m_j)}}\\ &< \mu_{j}+\Delta_j-2\sqrt{\frac{\log(T\tilde{\Delta}_{m_j}^{2})}{2n(m_j)}} \\&= \mu^*-2\sqrt{\frac{\log(T\tilde{\Delta}_{m_j}^{2})}{2n(m_j)}}\\& \le \bar{f}^*(m_j)-\sqrt{\frac{\log(T\tilde{\Delta}_{m_j}^{2})}{2n(m_{j})}},
\end{align*}
and action $j$ is eliminated. Hence, the probability that action $j$ is not eliminated in round $m_{j}$ is the probability that either one of the inequalities in~(\ref{eq:sub1}) do not hold. Using Chernoff-Hoeffding bound (Lemma~\ref{lem:chernoff}), we can bound this as follows, 
\begin{align}
\label{eq:cbl}
\bP\left[\bar{f}_{j}(m)> \mu_j+\sqrt{\frac{\log(T\tilde{\Delta}_{m}^{2})}{2n(m)}}\right] &\le \frac{1}{T\tDelta_m^2}\\\label{eq:cbu}
\bP\left[\bar{f}^*(m)< \mu^*-\sqrt{\frac{\log(T\tilde{\Delta}_{m}^{2})}{2n(m)}}\right]&\le \frac{1}{T\tDelta_m^2}.
\end{align}
Summing the above two inequalities for $m=m_{j}$ gives us that the probability that action $j$ is not eliminated in round $m_j$ by $*$ is at most $\frac{2}{T\tDelta_{m_j}^2}.$
\end{proof}
The next proposition is a modified version of Theorem~2 in~\cite{lairobbins}. We use it to obtain the regret lower bound in Proposition~\ref{prop:lbSO}.
\begin{proposition}\label{prop:lailb}
Suppose Assumptions~\ref{eq:klcond1},~\ref{eq:klcond2}, and~\ref{eq:condo} hold. Let $M_{i}(t)$ be the total number of observations for such a base-arm $i,$ for which $\vec{\theta} \in \Theta_{i}.$ Then, under any uniformly good policy $\bm{\phi}$, we have that
$$\liminf_{t\rightarrow\infty}\frac{\mathbb{E}[M_{i}(t)]}{\log(t)}\ge\frac{1}{J_{i}(\theta_{i})}.$$
\end{proposition}
\begin{proof}
By definition of $J_{i}(\theta_{i}),$ for every $\epsilon>0,$ there exists a $\theta'_{i}\in \mathcal{B}_{i}(\theta_{i})$ such that $J_{i}(\theta_{i}) < D(\theta_{i}||\theta'_{i})<(1+\epsilon)J_{i}(\theta_{i}).$ 

Now, under $\vec{\theta}'_{i}=[\theta_{1},\ldots,\theta'_{i},\ldots\theta_{N}],$ there exists an action $k\in \mathcal{S}_{i}$ such that $k$ is the unique optimal action. Then, for any uniformly good policy, for $0<b<\delta,$ $$\mathbb{E}_{\vec{\theta}'_{i}}[t-T_{k}(t)]=o(t^{b})$$ and therefore, $$\mathbb{P}_{\vec{\theta}'_{i}}\left[T_{k}(t)<(1-\delta)\log(t)/D(\theta_{i}||\theta'_{i})\right]=o(t^{b-1}),$$ similar to the asymptotic lower bound proof in~\citet{lairobbins}. 

Let $M_{i}(t)$ be the total number of observations for base-arm $i.$ Then $M_{i}(t)\geq T_{k}(t),$ since choosing any action in $\mathcal{S}_{i}$ gives observations for $i.$ Hence, 
 $$\mathbb{P}_{\vec{\theta}'_{i}}\left[M_{i}(t)<(1-\delta)\log(t)/D(\theta_{i}||\theta'_{i})\right]=o(t^{b-1}),$$ 
Now the rest of the proof of Theorem~2 in~\cite{lairobbins} applies directly to $M_{i}(t).$ We will repeat it below for completeness. Let $(Y_{i}(r))_{r\ge1}$ be the observations drawn from distribution $F_{i}$ and define $$L_{m}=\sum_{r=1}^{m}\log\left(\frac{g(Y_{i}(r);\theta_{i})}{g(Y_{i}(r);\theta'_{i})}\right).$$ Now, we have that $\mathbb{P}_{\vec{\theta'_{i}}}[C_{t}]=o(t^{b-1})$ where $C_{t}=\{M_{i}(t)<(1-\delta)\log(t)/D(\theta_{i}||\theta'_{i}) \mbox{ and } L_{M_{i}(t)}\le (1-b)\log(t)\}.$

Now, we use the change of measure arguments.
\begin{align}
\mathbb{P}_{\vec{\theta'_{i}}}&[M_1(t)=m_1,\ldots,M_N(t)=m_N,L_{m_i}\le (1-b)\log(t)]\\
&=\int_{\{M_1(t)=m_1,\ldots,M_N(t)=m_N,L_{m_i}\le (1-b)\log(t)\}}\Pi_{r=1}^{m_i}\frac{g(Y_{i}(r);\theta'_{i})}{g(Y_{i}(r);\theta_{i})}dP_{\vec{\theta_{i}}}\\
&\geq \exp(-(1-b)\log(t))\mathbb{P}_{\vec{\theta_{i}}}[M_1(t)=m_1,\ldots,M_N(t)=m_N,L_{m_i}\le (1-b)\log(t)]
\end{align}
Since $C_t$ is a disjoint union of events of the form $\{M_1(t)=m_1,\ldots,M_N(t)=m_N,L_{m_i}\le (1-b)\log(t)\}$ with $m_i<(1-\delta)\log(t)/D(\theta_{i}||\theta'_{i})$, it follows that
$$\mathbb{P}_{\vec{\theta}}[C_{t}] \le t^{1-b}\mathbb{P}_{\vec{\theta'_{i}}}[C_{t}]\rightarrow 0.$$

So far, we show that the probability of the event $C_t$ goes to $0$ as $t$ goes to infinity. If we show the event $\{ L_{M_{i}(t)}\le (1-b)\log(t)|M_{i}(t)<(1-\delta)\log(t)/D(\theta_{i}||\theta'_{i})\}$ occurs almost surely, then we show the probability of $\{M_{i}(t)<(1-\delta)\log(t)/D(\theta_{i}||\theta'_{i})\}$ goes to $0$ as $t$ goes to infinity, which is the desired result. By strong law of large numbers $L_{m}/m\rightarrow D(\theta_{i}||\theta'_{i})$ as $m\rightarrow \infty$ and $\max_{r\le m}L_{r}/m\rightarrow D(\theta_{i}||\theta'_{i})$ almost surely. Now, since $1-b > 1-\delta,$ it follows that as $t\rightarrow \infty,$  \begin{align}\mathbb{P}_{\vec{\theta}}&\left[L_{r}>(1-b)\log(t) \mbox{ for some } r <(1-\delta)\log(t)/D(\theta_{i}||\theta'_{i})\right] \rightarrow 0.\end{align}

Hence, we have that as $t\rightarrow \infty,$ $$\mathbb{P}_{\vec{\theta}}\left[M_{i}(t)<(1-\delta)\log(t)/D(\theta_{i}||\theta'_{i})\right]\rightarrow 0.$$
By choosing $\epsilon, \delta$ appropriately, this translates to $$\liminf_{t\rightarrow\infty}\frac{\mathbb{E}[M_{i}(t)]}{\log(t)}\ge\frac{1}{J_{i}(\theta_{i})}.$$
\end{proof}

\begin{proposition}\label{app:unknown}
When the horizon is unknown, start the UCB-LP algorithm with $\tilde{T}_0=2$ and increase $\tilde{T}$ after reaching $\tilde{T}$ steps by setting $\tilde{T}_{l+1}=\tilde{T}_l^2.$ The regret of unknown horizon UCB-LP is bounded by
\begin{equation}
\sum_{j\in \cU\setminus B}\frac{64\Delta_{j} z_{j}^{*}}{\hat{\Delta}_{j}^{2}} \log(T\hat{\Delta}_{j}^{2})+\sum_{j\in B} \frac{64\log(T\Delta_{j}^{2})}{\Delta_{j}}+O(K^2\log_2\log_2 T).
\end{equation}
\end{proposition}
\begin{proof}
When the horizon is unknown, start the UCB-LP algorithm with $\tilde{T}_0=2$ and increase $\tilde{T}$ after reaching $\tilde{T}$ steps by setting $\tilde{T}_{l+1}=\tilde{T}_l^2.$ Thus, $\tilde{T}_l=2^{2^l}$ until reaching horizon T. Also, the period in which horizon is reached is denoted by L. Note that $2\leq L\leq \log_2\log_2 T$.

In any period $l, (0\leq l \leq L)$, UCB-LP uses $\tilde{T}_l$ as input. Note that $\bar{m}, m_j, B$ and $\hat{\Delta}_j$ are independent of $\tilde{T}_l$, thus $l$. Recall that regret of UCB-LP is bounded by (\ref{eq:ucblp}). The regret of UCB-LP with unknown horizon is upper bounded by the summation over all the periods.

\begin{align*}
\sum_{l=0}^L \left[\sum_{j\in \cU\setminus B} \Delta_{j} z_{j}^{*}\frac{32\log(\tilde{T}_l\hat{\Delta}_{j}^{2})}{\hat{\Delta}_{j}^{2}} + \sum_{j\in B} \frac{32\log(\tilde{T}_l\Delta_{j}^{2})}{\Delta_{j}}+O(K^{2})\right] &= (i)+(ii)+(iii).
\end{align*}

First, we consider the term (i). We can plug in the definition of $\tilde{T}_l$ into (i).
\begin{align*}
(i)&=\sum_{l=0}^L \sum_{j\in \cU\setminus B} \Delta_{j} z_{j}^{*}\frac{32\log(\tilde{T}_l\hat{\Delta}_{j}^{2})}{\hat{\Delta}_{j}^{2}}\\
&=\sum_{j\in \cU\setminus B}\frac{32\Delta_{j} z_{j}^{*}}{\hat{\Delta}_{j}^{2}}\sum_{l=0}^L\log(2^{2^l}\hat{\Delta}_{j}^{2})\\
&=\sum_{j\in \cU\setminus B}\frac{32\Delta_{j} z_{j}^{*}}{\hat{\Delta}_{j}^{2}} \left((\log 2)\sum_{l=0}^L{2^l}+(L+1)\log\hat{\Delta}_{j}^{2}\right)\\
&\leq\sum_{j\in \cU\setminus B}\frac{32\Delta_{j} z_{j}^{*}}{\hat{\Delta}_{j}^{2}} \left(2^{L+1}(\log 2)+(L+1)\log\hat{\Delta}_{j}^{2}\right)\\
&\leq\sum_{j\in \cU\setminus B}\frac{32\Delta_{j} z_{j}^{*}}{\hat{\Delta}_{j}^{2}} \left(2\log T+(L+1)\log\hat{\Delta}_{j}^{2}\right) ~~~(since~ L\leq\log_2\log_2 T)\\
&\leq\sum_{j\in \cU\setminus B}\frac{64\Delta_{j} z_{j}^{*}}{\hat{\Delta}_{j}^{2}} \log(T\hat{\Delta}_{j}^{2})~~~(since~(L+1)\log\hat{\Delta}_{j}^{2}\leq2\log\hat{\Delta}_{j}^{2})
\end{align*}

Similarly, we have that $(ii)\leq \sum_{j\in B} \frac{64\log(T\Delta_{j}^{2})}{\Delta_{j}}$. Now, we directly sum up the bound for term (iii).

\begin{equation*}
(iii)\leq\sum_{l=0}^L O(K^2)\leq(L+1)O(K^2)=O(K^2\log_2\log_2 T).
\end{equation*}

Hence, by combining the results above, the regret of unknown horizon is bounded by
\begin{equation*}
\sum_{j\in \cU\setminus B}\frac{64\Delta_{j} z_{j}^{*}}{\hat{\Delta}_{j}^{2}} \log(T\hat{\Delta}_{j}^{2})+\sum_{j\in B} \frac{64\log(T\Delta_{j}^{2})}{\Delta_{j}}+O(K^2\log_2\log_2 T).
\end{equation*}
\end{proof}



\end{document}